\documentclass{article}

\usepackage{microtype}
\usepackage{subfigure}
\usepackage{booktabs} 

\usepackage[accepted]{icml2020}
\usepackage{hyperref}

\icmltitlerunning{Maximum-and-Concatenation Networks}

\usepackage[utf8]{inputenc} 
\usepackage[T1]{fontenc}    
\usepackage{bm}       		
\usepackage{url}            
\usepackage{booktabs}       
\usepackage{amsfonts}       
\usepackage{amsmath, amsthm, amssymb}
\usepackage{nicefrac}       
\usepackage{microtype}      
\usepackage{enumitem}
\usepackage{graphics,graphicx}
\usepackage{threeparttable}
\usepackage{mathtools}
\usepackage{physics}

\newtheorem{theorem}{Theorem}
\newtheorem{corollary}{Corollary}[theorem]
\newtheorem{lemma}[theorem]{Lemma}
\newtheorem{claim}{Claim}
\newtheorem{condition}{Condition}

\def\*#1{\mathbf{#1}}
\def\^#1{\mathcal{#1}}
\def\##1{\mathbb{#1}}
\newcommand{\RNum}[1]{\lowercase\expandafter{\romannumeral #1\relax}}

\begin{document}

\twocolumn[
\icmltitle{Maximum-and-Concatenation Networks}

\begin{icmlauthorlist}
	\icmlauthor{Xingyu Xie}{pku}
	\icmlauthor{Hao Kong}{pku}
	\icmlauthor{Jianlong Wu}{sdu}
	\icmlauthor{Wayne Zhang}{sensetime}
	\icmlauthor{Guangcan Liu}{nuist}
	\icmlauthor{Zhouchen Lin}{pku}
\end{icmlauthorlist}

\icmlaffiliation{pku}{Key Lab. of Machine Perception (MoE), School of EECS, Peking University}
\icmlaffiliation{sdu}{School of Computer Science and Technology, Shandong University}
\icmlaffiliation{sensetime}{SenseTime Research}
\icmlaffiliation{nuist}{B-DAT and CICAEET, School of Automation, Nanjing University of Information Science and Technology}

\icmlcorrespondingauthor{Guangcan Liu}{gcliu@nuist.edu.cn}
\icmlcorrespondingauthor{Zhouchen Lin}{zlin@pku.edu.cn}



\vskip 0.3in
]

\printAffiliationsAndNotice{}  
\begin{abstract}
While successful in many fields, deep neural networks (DNNs) still suffer from some open problems such as bad local minima and unsatisfactory generalization performance. 
In this work, we propose a novel architecture called Maximum-and-Concatenation Networks (MCN) to try eliminating bad local minima and improving generalization ability as well.
Remarkably, we prove that MCN has a very nice property; that is, \emph{every local minimum of an $(l+1)$-layer MCN can be better than, at least as good as, the global minima of the network consisting of its first $l$ layers}. 
In other words, by increasing the network depth, MCN can autonomously improve its local minima's goodness, what is more, \emph{it is easy to plug MCN into an existing deep model to make it also have this property}.
Finally, under mild conditions, we show that MCN can approximate certain continuous functions arbitrarily well with \emph{high efficiency}; that is, the covering number of MCN is much smaller than most existing DNNs such as deep ReLU. 
Based on this, we further provide a tight generalization bound to guarantee the inference ability of MCN when dealing with testing samples. 
\end{abstract}
\section{Introduction}
\vspace{-2mm}
Deep neural networks (DNNs) have been showing superior performance in various fields such as computer vision, speech recognition, natural language processing, and so on. At the first glance, DNN learning is not an enigmatic technique, as its basic idea is quite simple and mostly about learning a possibly over-parameterized DNN from a huge number of training samples; namely,
\begin{equation}\label{eq:obj0}
\vspace{-2mm}
\min_{\bm{\theta}} L(\bm{\theta}) := \frac{1}{n}\sum_{i=1}^n\ell(\*f_{\bm{\theta}}(\*x_i), \*y_i),
\end{equation}
where $\*x_i \in \mathbb{R}^{d_x}$ and $\*y_i \in \mathbb{R}^{d_y}$ denote an input and a target, respectively, $\*f_{\bm{\theta}}(\cdot)$ standards for a DNN with parameters $\bm{\theta}$, and $\ell(\cdot)$ is some loss function. Notice that, some kind of regularization schema has already been implanted into the network to constrain the parameter space, though there is no explicit regularizer imposed on $\bm{\theta}$~\cite{arora2019implicit}. Despite its ordinary appearance, DNN learning is meanwhile quite complicated in many ways, and the current DNNs still suffer from several weaknesses, e.g., the training procedure may easily get stuck in \emph{bad local minima} (i.e., the local minima with large training error), the learnt model may be prone to \emph{over-fit the training data} (i.e., the testing error is large when small training error is obtained), etc. Overcoming these difficulties are crucial for DNNs to solve the real-world problems that are more challenging and significant, but they are still \emph{open} problems.

To address the issue of bad local minima, many heuristic techniques have been proposed, e.g., batch normalization~\cite{ioffe2015batch}, group normalization~\cite{wu2018group}, dropout~\cite{srivastava2014dropout}, etc. These techniques would be useful under certain context, but may not be generally helpful and, even worse, it is hard to know when and which method should be used. In fact, the elimination of bad local minima, i.e., having small empirical training error at \emph{all} local minima, is really important for DNN learning. Some recent theories~\cite{zhang2016understanding,wei2019data,cao2019generalization,li2018learning,allen2018learning,arora2019fine} have revealed that, whenever the local minima produces only small training error, DNNs have probably good generalization performance at these local minima. That is to say, in some cases, good local minima mean good predictors which are the ultimate goal of supervised learning. With the hope of pursuing the property of \emph{no bad local minima}, some learning theories~\cite{kawaguchi2016deep, arora2018convergence, hardt2016identity,liang2018adding,liang2018understanding} have been established to prove that, under certain conditions, any local minima of a certain DNN are also global minima. While impressive, existing studies are still unsatisfactory in some aspects:
\begin{itemize}
\item Most existing theories about ``all local minima are global minima'' are built upon on some unrealistic network architectures, e.g., without activation function, which means that they cannot be applied to common deep learning tasks. 
The work~\cite{kawaguchi2019elimination} considers general architectures, but requires additional regularizer and is limited to shallow case. In addition, strictly speaking, the conclusion of ``all local minima are global minima'' cannot really ensure that ``DNN has no bad local minima''. This is because, whenever the adopted network itself is poorly designed, global minima can still lead to large training error. In one word, existing studies have not gained convenient schemes that can be easily used to reduce the training error of general DNNs.
\item Though small training error may bring good generalization for some specially designed DNNs~\cite{zhang2016understanding,wei2019data,cao2019generalization,li2018learning,allen2018learning,arora2019fine}, a rigorous generalization bound is still important for general DNNs to produce superior performance in practice. There is sparse research in the direction of generalization analysis, e.g., deep ReLU~\cite{yarotsky2017error}. However, the covering number in deep ReLU is very large, which means that the approximation ability of the network is rather weak.
\item  What is more, to our knowledge, there is no theoretical study that addresses the issues of local minima and generalization ability simultaneously. These two problems are closely related and should be investigated at the same time.
\vspace{-2mm}
\end{itemize}
To relieve the issues highlighted above, we propose a novel multi-layer DNN termed \textbf{M}aximum-and-\textbf{C}oncatenation \textbf{N}etworks~(MCN).
In our MCN, one hidden layer is formed by concatenating together two parts, with one being a linear transformation of the output of the previous layer, and the other being a maximum of two piecewise smooth functions. The output of the final layer is further transformed by some linear operators, so as to stay in step with the configuration of the target output. In general, the concatenation operator is a good option during designing DNNs, and it is indeed a primary cause of the superiorities of MCN over existing architectures.
\par
We prove that MCN naturally ensures the effectiveness of its learning process, i.e., the no bad local minima property. To be more precise, suppose that $\bm{\theta}'$ is a global minimum to (\ref{eq:obj0}) with $\*f_{\bm{\theta}'}$ being an $l$-layer MCN (briefly, we say that $\bm{\theta}'$ is a global minimum of an $l$-layer MCN), and $\bm{\theta}$ is a local minimum of the $(l+1)$-layer MCN obtained by adding one layer to the former $l$-layer network. Then we have $ L(\bm{\theta}) \leq L(\bm{\theta}')$, which means that \emph{the global minima of an $l$-layer MCN may be outperformed, at least can be attained, by simply increasing the network depth}. More importantly, MCN can be easily appended to many existing network architectures, and we prove that, under mild conditions, \emph{the modified DNN will get the nice properties of MCN}. This property is achieved mainly due to a \emph{skip connection} with a proper activation function: With the help of this skip connection, the bad local minima are moved to infinity, while the implicit regularizer carried by the network itself may encourage the optimization procedure to seek for the remaining good local minima.
\par
\begin{figure*}[!htp]
	\centering
	\includegraphics[width=\textwidth]{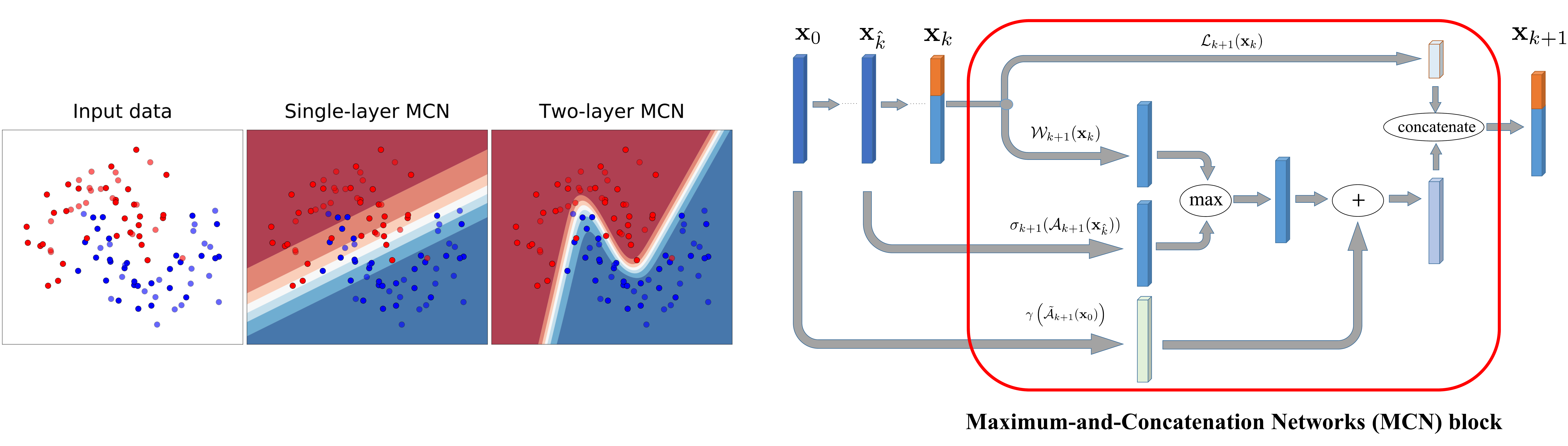}
	\caption{Left: Illustration of the motivations for inventing MCN, which is indeed a generalization of the piecewise smooth function. The composition of MCNs may increase the pieces exponentially. Right: One block of MCN, where each layer consists of four parts.}
	\label{fig:LM-CVNN}
\end{figure*}
Notice that, piecewise liner functions can approximate any Lipschitz continuous function up to arbitrarily small error, and the maximum operator can model the piecewise linear function efficiently~\cite{telgarsky2016benefits}. Based on these facts, we show that MCN with \emph{sparse} connection can approximate a wide range of continuous functions arbitrarily well. Our analysis framework is new and quite different from the previous studies~\cite{lu2020deep,yarotsky2018optimal,yarotsky2017error}, which rely on Taylor expansion and requires a parameter complexity of $\^O(N^{d_x})$, where $N\gg 1$ is a quantity that controls the approximation accuracy~\footnote{$N^{-\beta}$ is the dominant term in approximation error, where $\beta >0$ relates to the smoothness of the target function.}. By sharp contrast, we show that a complexity of only $\^O(N(\ln N)^{d_x -2})$ is enough to approximate the target function.
\par
Based on the approximation analysis, we further investigate the generalization ability of MCN to cope with testing samples, proving that MCN has much smaller covering number than deep ReLU. Interestingly, our results suggest that the width has less effects than the depth on the generalization bound. Our results also show that, whenever the training data are exactly fitted, \emph{MCN achieves the statistically optimal rate in the minmiax sense}; this confirms the conjectures in~\cite{wei2019data,arora2019fine,belkin2018overfitting} that ultra-deep networks may generalize well on testing data~\footnote{Note here that we have no intention to suggest using infinitely deep networks, as the computational cost is also a matter and the required data amount in the extreme case could be huge.}. To summarize, the contributions of this paper mainly include: 	
\vspace{-3mm}
\begin{itemize}
	\item
	We propose a novel architecture termed MCN and prove that MCN can help to overcome the issue of bad local minima. Namely, the global minima of an $l$-layer MCN can be always attained or even outperformed by simply increasing the network depth (Theorem \ref{thm:depth}). More importantly, we show that MCN is able to turn a possibly poorly-designed DNN into a good one, which also has the nice property of ``no bad local minima" under certain conditions (Corollary \ref{corollary:part} and Theorem \ref{thm:full}). These results would be more significant than~\cite{kawaguchi2016deep, hardt2016identity}, which only show that all local minima of a certain DNN with fixed depth are global minima, but provide no practical guidance for the users to seek better solutions to their tasks---just finding the globally optimal solutions to some over-simplified optimization problems is essentially not enough.
	
	\item We devise a new framework to analyze the approximation ability of MCN, showing that MCN can approximate some classes of continuous functions arbitrarily well by only using a parameter complexity of $\^O(N(\ln N)^{d_x -2})$ (Theorem~\ref{thm:approximation}). This is much lower than the $\^O(N^{d_x})$ complexity obtained by the previous studies~\cite{lu2020deep,yarotsky2018optimal,yarotsky2017error}.

	\item Unlike the previous analyses in~\cite{liang2018adding,liang2018understanding,kawaguchi2019elimination}, which focus on the elimination of local minima but ignore the generalization performance, we provide rigorous analysis to guarantee the generalization ability of MCN under certain conditions (Theorem \ref{thm:covering} and Corollary \ref{corollary:risk}). In particular, our results show that MCN has a much smaller covering number than deep ReLU, revealing that the depth is more important than the width for generalization; this supports the mechanism of deep learning.
	
\end{itemize}
\vspace{-2mm}
\section{Model and Setting}
This section introduces the technical details of MCN, as well as the setup for establishing theoretical analysis.
\vspace{-2mm}
\subsection{Maximum-and-Concatenation Networks}
The design of our MCN---a linearity and maximum concatenation network---is inspired by the following observations. Consider the task of shattering some points that are not linearly separable, which is shown in Figure~\ref{fig:LM-CVNN}. Intuitively, the maximum of two hyperplanes may produce smaller classification error than every single one of them. Therefore, we may reduce the classification error by replacing parts of the current classifier with some maximum-derived units. Such a replacement process can be repeated several times, learning progressively a refined classification surface that will be piecewise smooth. Moreover, considering the regression problem, we have a classical claim from the Stone-Weierstrass approximation theorem.
\begin{claim}\label{claim:lipfrompiece}
	Any Lipschitz continuous function can be approximated arbitrarily well by a piecewise linear function.
\end{claim}
\vspace{-2mm}
By composing a series of maximum operators, we can easily construct a piecewise smooth function. Consider approximating the quadratic function $x\rightarrow x^2$. Define the operator $\^T^m(x) \coloneqq \max\{-x/2,x/2 - 2^{1-2m}\}$ and let $g^m(x) \coloneqq \^T^m\circ \^T^{m-1}\circ \cdots \^T^1(x)$. It is known that $x+\sum_{i=1}^m g^i(x)$ approximates $x^2$ exponentially fast in $m$~\cite{telgarsky2016benefits}. In contrast, to approximate a twice differentiable non-piecewise linear function $f$, it would be awkward to use some existing DNNs that need to rescale the second order differences: $(f(t+2\delta x) -2f(t+x\delta)+f(x\delta))/(\delta^2 f''(t)) \rightarrow x^2$ for $\delta\to 0$ with $f''(t) \ne 0$. Note that $\delta\to 0$ will cause the scale of network parameters to be very large.
\par
Beneath it all, the model of an $l$-layer MCN, which is indeed a mapping from input $\*x$ to output $\*y$, is designed as follows, for $k=0,\cdots,l-1$:
\begin{equation}\label{Eq:LM-CVNN}
\*x_{k+1} = \left[	\^L_{k+1}(\*x_k);~\gamma\left(\^{\tilde{A}}_{k+1}(\*x_0)\right) + \^M_{k+1}(\*x_k)\right],
\end{equation}
where
\[
\^M_{k+1}(\*x_k)=\max\left\{\^W_{k+1}(\*x_k),  \sigma_{k+1}\left(\^A_{k+1}\left(\*x_{\hat{k}}\right)\right)\right\},
\]
$0\leq\hat{k}\leq k$ ($\*x_{\hat{k}}$ is the output of any intermediate layer between $\*x_{k}$ and $\*x_{0}$), $\gamma(\cdot)$ and $\sigma_{k+1}(\cdot)$ are some element-wise activation functions, $\*x_0 = \*x\in \mathbb{R}^{d_x}$ is the input data vector, $\*x_k\in \mathbb{R}^{d_k}$ is the output of the $k$-th layer, $[~;~]$ is the operator that vertically concatenates two vectors into a single one, $\^L_{k+1}:\mathbb{R}^{d_k} \to \mathbb{R}^{d_{\^L}}$ is a learnable linear operator\footnote{For convenience, we assume that the output of $\^L_{k+1}$ has a fixed dimension $d_{\^L}$, $\forall{}k=0,\cdots,l-1$. Actually, our methods and theories do not need this assumption.}, and $\^A_{k+1}(\cdot)$, $\^{\tilde{A}}_{k+1}(\cdot)$ and $\^W_{k+1}(\cdot)$ are all learnable linear operators from $\mathbb{R}^{d_k}$ to $\mathbb{R}^{d_{k+1} - d_{\^L}}$.
\par
In fact, as mentioned in Figure~\ref{fig:LM-CVNN}, MCN is a generalization of piecewise smooth functions, and it can contain many existing DNNs as special cases, e.g., ResNet, Maxout Network~\cite{goodfellow13} and Input Convex Neural Networks (ICNN)~\cite{amos2017input}. In MCN, there are layers that directly connect the input $\*x_0$ to the hidden units in deeper layers. Such connections are unnecessary for traditional networks, but very important for achieving the nice property of ``no bad local minimum'' which we will introduce later. The highway with the operator $\^L_k$ connects the training loss with the geometric projection residual in the proper setting (Section \ref{sec:connectLR} in supplementary material), which helps MCN perform well when it goes deeper and wider.

\subsection{Setting}
To analyze MCN theoretically, we consider a typical task of regression (or classification). Denote by $\*x \in \mathbb{R}^{d_x}$ and $\*y \in \mathbb{R}^{d_y}$ an input vector and a target, respectively. Let $\{(\*x_i, \*y_i)\}_{i=1}^n$ be a training set consisting of $n$ samples, with $\{\*x_i\}_{i=1}^n$ being \emph{distinct} points in $\mathbb{R}^{d_x}$. Denote by $\*x_{k,i}$ the output of the $k$-th layer on the $i$-th training sample $\*x_i$. Notice that MCN is primarily designed to learn some extrinsic structures from the data $\*x$, and its outputs may be inconsistent with the target $\*y$, e.g., they might have different dimensions. Hence, an additional mapping $\Psi:\mathbb{R}^{d_l}\rightarrow \mathbb{R}^{d_y}$ is used to further transform the network outputs, resulting in the following objective function for training an $l$-layer MCN:
\begin{equation}\label{eq:tr_loss}
\vspace{-2mm}
L(\bm{\theta}_{l}) := \frac{1}{n}\sum_{i=1}^n\ell(\Psi(\*x_{l,i}), \*y_i),
\vspace{-1mm}
\end{equation}
where $ \ell: \mathbb{R}^{d_y} \times \mathbb{R}^{d_y} \rightarrow \mathbb{R}$ is an arbitrary lower-bounded loss function (without losing generality, we assume the lower bound is 0), and $\bm{\theta}_{l} = \{\bm{\theta} (\^L_{k}, \^W_{k}, \^{\tilde{A}}_{k}, \^A_{k})\}_{k=1}^l$ is a collection of all learnable parameters with $ \bm{\theta}(\^L_{k}, \^W_{k}, \^{\tilde{A}}_{k}, \^A_{k})$ being the parameters of the operators $\^L_{k}, \^W_{k}, \^{\tilde{A}}_{k}$ and $\^A_{k}$ defined in (\ref{Eq:LM-CVNN}). In our setup, the extra mapping $\Psi(\cdot)$ could be either learnt or fixed~\footnote{There is no much difference between these two variants, as fixing the last layer of a DNN may cause very little influence~\cite{hoffer2018fix}.}, while the activation functions $\sigma_{k}(\cdot)$ and $\gamma(\cdot)$ are always fixed.

To obtain rigorous conclusions, some technical conditions are required. But for the ease of presentation, we would like to present them along with the established theorems.  
\section{Main Results}
This section presents the main results of this paper, including a couple of theories regarding the optimality, fitting ability and generalization performance. All the detailed proofs of these theorems are provided in the supplementary material.

\subsection{Effects of Depth}
First note that an $(l+1)$-layer MCN is obtained by adding one layer into the network consisting of its first $l$ layers, i.e., $\bm{\theta}_{l+1} = \{\bm{\theta}_{l}, \bm{\theta} (\^L_{l+1}, \^W_{l+1}, \^{\tilde{A}}_{l+1}, \^A_{l+1})\}$. Under some mild technical conditions, we prove that the training objective~\eqref{eq:tr_loss} is non-increasing, or even monotonically decreasing, as the network goes deeper\footnote{This is not in conflict with the learning-based optimization theories~\cite{xie2019DLADMM,liu2018alista}, which show that their networks can converge fast and need only a smaller number of layers to solve optimization problems. In fact, empirically, MCN will converge when the network is deep enough.}.
\begin{theorem}[Effects of Depth]\label{thm:depth}
Let the activation function $\gamma(\cdot)$ be the element-wise $\exp(\cdot)$. Suppose that the loss function $\ell(\cdot)$ in~(\ref{eq:tr_loss}) is differentiable and convex. Denote by $\bm{\theta}_{l+1}$ any local minimum of an $(l+1)$-layer MCN. If $d_{l+1} = d_l$, then the following holds for any fixed injection $\Psi(\cdot)$:
	\[
	L(\bm{\theta}_{l+1})  \leq \min_{\bm{\theta}'_{l}} L(\bm{\theta}'_{l}),
	\]
	where $\bm{\theta}'_{l}$ is a global minimum of the $l$-layer MCN. Moreover, if $\ell(\cdot)$ is strongly convex and there exists $i$ such that $\*x_{l+1,i}\neq \*x'_{l,i}$, then the inequality is strict, namely $L(\bm{\theta}_{l+1})  < \min_{\bm{\theta}'_{l}} L(\bm{\theta}'_{l})$.
\end{theorem}
The setting of fixing $\Psi(\cdot)$ is to ensure that an $(l+1)$-layer MCN and its $l$-layer part are comparable. According to the above theorem, the global minima of an $l$-layer MCN can be attained, or even outperformed, by simply increasing the network depth by one. So, given the context of MCN, increasing network depth can not only ``eliminate'' local minima, but also help seek good solutions that possess smaller training error, providing a theoretically interpretation for a well-known empirical observation---deeper networks usually lead to better training results.

Among the other things, provided that the loss function is differentiable and strongly convex, we can further prove that the training error is able to go to zero. But the proof needs a key theorem established in the next subsection.
\par
\textbf{Remark 1}: One may worry that there exist decreasing paths to infinity, and the weight may need to diverge to improve the performance of local minima~\cite{sohl2019eliminating}.
The previous work~\cite{kawaguchi2016deep,liang2018adding,liang2018understanding} may suffer from this problem, mainly due to their explicit regularization, whose coefficient should decay to zero to ensure the consistency of optimization. 
Hence, it leads to the divergence of some parameters to ensure the scale of output.
However, our results hold without requiring any parameter to approach zero or infinity.
Furthermore, for the classification problem, this divergence problem can be solved by proper parameter regularization~\cite{liang2019revisiting}. 
But, for the general regression problem, regularization may not work.
Fortunately, under the over-parameterized setting, algorithmic analysis~\cite{allen2019convergence,du2018gradient1} can entirely avoid the divergence risk. 
We leave the algorithmic analysis of MCN as our future work.
\subsection{Approximation Ability}
In general, it is unlikely that all mathematical functions can be approximated by DNNs. The following defines a class of functions which can be well approximately by MCN.
\begin{condition}\label{cond:r2f}
	For $\beta \in \#N$, we define a modified $\beta$-th Sobolev space on the hypercube $[-1,1]^{d_x}$
	\[
	\^H^{\beta}\coloneqq\left\{\*f: \mathrm{D}^{\bm{\alpha}} \*f \in \mathrm{L}^{2}\left(\left[-1,1\right]^{d_x}\right), \forall \bm{\alpha}:|\bm{\alpha}|_{\infty} \leq \beta\right\},
	\]
	where $\bm{\alpha} = \left(\alpha_1,\cdots,\alpha_{d_x}\right) \in \#N^{d_x}$ is a multi-index, $\mathrm{D}^{\bm{\alpha}}$ corresponds to the  \textbf{weak} derivatives  operator $\partial_{x_{1}}^{\alpha_{1}} \ldots \partial_{x_{d_x}}^{\alpha_{d}}$ of order $|\bm{\alpha}| = \alpha_1+\cdots+\alpha_{d_x}$ and $|\bm{\alpha}|_{\infty} = \max\{\alpha_i\}$. It is assumed that the function $\*f \in \^H^{2\beta+2}$ obeys the homogeneous Neumann boundary conditions up to order $\beta$:
	\[
	\left.\partial_{x_{j}}^{2 r+1} \*f\right|_{\partial \Omega_{j}}=0, \quad j=1, \ldots, d_x, \quad r=0, \ldots, \beta-1,
	\]
	where $\partial \Omega_{j}=\left\{\*x \in[-1,1]^{d}: x_{j}=\pm 1\right\}$ is the boundary.
\end{condition}
The above condition depicts a class of continuous functions $\*f \in \mathrm{L}^{2}([-1,1]^{d_x})$ such that $\*f$ and its weak derivatives up to a certain order have finite $L_2$ norm. Note that the Neumann boundary condition of $[-1,1]^{d_x}$ is not harsh, and we can always extend the target function by firstly using the Sine or Cosine functions to introduce the homogeneous Neumann property and then scaling it to the interval $[-1,1]^{d_x}$.

As pointed out by~\cite{barron1993universal}, a standard fully connected neural network with enough, possibly infinite, hidden units can approximate any continuous function in compact domain. For MCN, we have an explicit approximation bound to connect the width and depth in a finite fashion.
\begin{theorem}[Approximation Ability]\label{thm:approximation}
	Let $\*f$ be a vector-valued function that obeys Condition \ref{cond:r2f}, and let $w\geq0,p\geq 0,N\gg 1$ be given numbers. Define $N_d = N \left(\ln N\right)^{d_x -2}$, and denote by $\*f_{\bm{\theta}}$ the output of an MCN. Suppose either $\*f_{\bm{\theta}}$ is of width $\^O(N_d d_x wp\ln p)$ and depth $\^O(l\ln p + N^2)$, or $\*f_{\bm{\theta}}$ has $\^O(d_x wp \ln p)$ width and $\^O(N_d l\ln p + N^2N_d)$ depth. Then $\*f$ can be approximated by MCN with proper parameters, in a sense that:
	\[
	\|\*f_{\bm{\theta}}(\*x) - \*f(\*x)\|_\infty \leq \epsilon,\quad \forall \*x \in [-1,1]^{d_x},
	\]
	where
	\[
	\epsilon =
	\^O\left(d_x2^{d_x}p^2 2^{-wl} + N^{-2\beta-2}\left(\ln N \right)^{d_x-1}\right).
	\]
The number of non-zero parameters in $\bm{\theta}$ is in the order of $\^O\left(N_d\left(d_xw^2p\ln p + N^2\right)\right)$.
\end{theorem}
\begin{proof}[Proof Sketch] We first construct the shallow MCNs that approximate $\sin(n\pi x)$ and $\cos(n\pi x)$ for different $n \in \#N$ exponentially fast. Then we can obtain a multivariate function $\phi_{\*n} \coloneqq \prod_{i=1}^{d_1} \sin(n_i\pi x)\prod_{k=d_1+1}^{d_x}\cos(n_k\pi x)$ by an MCN of $\^O(\ln d_x)$ depth, where $d_1 \leq d_x$ and $\*n \in \#N^{d_x}$. Since the set $\{\phi_{\*n}, \*n\in \#N^{d_x}\}$ is a Fourier orthogonal basis for $\mathrm{L}^{2}([-1,1]^{d_x})$, we can prove that $N(\ln N)^{d_x -2}$ sub-MCNs suffice to approximate the target function, where $N = \prod_i n_i$. More detailed proofs can be founded in the supplementary material.
\end{proof}

Remarkably, Theorem \ref{thm:approximation} shows that MCN requires only a parameter complexity of $\^O\left(d_xN(\ln N)^{d_x-2}\right)$ to approximate the target function, which is dramatically lower than the $\^O(d_x^{d_x}N^{d_x})$ required by deep ReLU~\cite{yarotsky2017error}. This is mainly benefited from our analysis techniques. Unlike the analyses in~\cite{lu2020deep,yarotsky2018optimal}, which split the input space into small hyper-cubes and use a local network to approximate the Taylor expansion on those hyper-cubes, our analysis is built upon high-dimensional Fourier expansions and can therefore obtain higher decay rate for the approximation residual. Besides, the special network architecture of MCN is another cause of the advantage of lower complexity. Namely, the maximum operator makes the power of the decay term for approximating underlying polynomial  be in the order of width$\times$depth. By contrast, the decay power is just proportional to the depth in deep ReLU.

In summary, Theorem \ref{thm:approximation} illustrates that MCN with highly sparse connectivity between neurons can produce good approximation performance. This forms good basis for establishing tight generalization bound and eliminating bad local minima, as will be shown soon.
\subsection{Generalization Bound}
Theorem \ref{thm:approximation} ensures the existence of a good predictor when MCN goes deeper and wider. Now, one natural question is: does the generalization bound also shrink as the network becomes deeper? To analyze the generalization ability of DNNs or any other learning methods, it is indeed necessary to make some assumptions about the data. In this subsection, we set $d_y =1$ and assume that $\*x_i \in [0,1]^{d_x}$ for $i = 1,\cdots,n$. We consider the nonparametric regression task, i.e., there exists a target oracle function $f_0$ such that
\begin{equation}\label{eq:reg}
	y_i = f_0(\*x_i) + \varepsilon_i,\quad i = 1,\cdots,n,
\end{equation}
where the noise terms $\varepsilon_i$'s are assumed to be i.i.d. Gassuian and independent of $\*x_i$.
\par
Denote the function class of our MCN as
\[
\^F(\bm{\theta}, s)\coloneqq \left\{f_{\bm{\theta}}: \operatorname{Supp}\left(\bm{\theta}\right)<s, \|\bm{\theta}_k\|^2_{F} < \infty, \forall k\leq l \right\},
\]
where $\|\bm{\theta}_k\|_{F} $ is the Frobenius norm of all the parameters at the $k$-th layer, and the operator $\operatorname{Supp}(\cdot)$ denotes the support of a set, i.e., $\operatorname{Supp}\left(\bm{\theta}\right)$ is the number of non-zero parameters in MCN. The boundness assumption of $\operatorname{Supp}\left(\bm{\theta}\right)<s$ is made on the basis of Theorem \ref{thm:approximation}, which shows that MCN with sparse connections can possess strong approximation ability. For convenience, we consider the case where the structure of $\^F(\bm{\theta}, s)$ is deterministic, i.e., the input layer of $\^A_k(\cdot)$ is the same for all MCNs in $\^F(\bm{\theta}, s)$. Denote by $\^N\left(\delta, \^F(\bm{\theta}, s), \|\cdot\|_1\right)$ the minimal number of $\ell_1$-balls with radius $\delta$ that covers $\^F(\bm{\theta}, s)$. The logarithm of $\^N\left(\delta, \^F(\bm{\theta}, s), \|\cdot\|_1\right)$ is also called the \emph{covering number} for convenience. For an operator $\^A$, $\|\^A\|_{1}$ denotes its $\ell_1$ norm induced by the vector $\ell_1$ norm, namely $\|\^A\|_{1} = \max_{\*x \ne 0} \frac{\|\^A(\*x)\|_{1}}{\|\*x\|_1}$. Then we have the following theorem to bound the covering number (i.e., $\ln \^N\left(\delta, \^F(\bm{\theta}, s), \|\cdot\|_1\right)$).
\begin{theorem}[Covering Number of MCN]\label{thm:covering}
	Assume that the activation function $\sigma_k(\cdot)$ is $\rho_k$-Lipschitz and $\rho_k\leq \rho$ for $k=1,\cdots,l$. Then one block of MCN is $\kappa_k$-Lipschitz continuous w.r.t. the input layers and
	\[
	\kappa_{k} = \left( 1  +  \max\{\rho_{k},2 \}\|\bm{\theta}_k\|_{1}  \right),
	\]
	where
	\[
	\|\bm{\theta}_k\|_{1} \coloneqq \max\{ \|\^{\tilde{A}}_{k}\|_{1}, \|\^A_{k}\|_{1}, \|\^W_{k}+\^L_{k}\|_{1} \}.
	\]
	Moreover, we have
	\[
	\ln \mathcal{N}\left(\^F(\bm{\theta},s), \delta,\|\cdot\|_{1}\right) \leq\^O\left(sl \ln\left(\frac{\rho  w \prod_{k=1}^l \kappa_k}{\delta}\right)  \right),
	\]
	where $w$ and $l$ are the width and depth of MCN, respectively.
\end{theorem}
The above theorem shows that the covering number of MCN is $\^O\left(sl^2 \ln \left(w/\delta\right)\right)$, where $s = \Theta\left(d_xN(\ln N)^{d_x-2}\right)$. By contrast, to achieve the same approximation accuracy, deep ReLU needs a covering number of $\^O\left(s'l\ln \left(s'w^2l/\delta\right)\right)$, with $s' = \Theta(d_x^{d_x}N^{d_x})$. In the situation of high-dimensional data, i.e., $d_x$ is large, it is clear that MCN has much smaller covering number than deep ReLU, which means that the model complexity of MCN is much lower. Due to this, MCN provably owns good generalization performance, as shown in the following.
\begin{corollary}[Generalization Bound]\label{corollary:risk}
	Consider the regression problem in (\ref{eq:reg}) and assume $\max_{\*x \in [0,1]^{d_x}} f_0(\*x)<\infty$. Let $f_{\operatorname{M}}$ be any MCN from $\^F(\bm{\theta}, s)$, and define
	\[
	\ell_n(f) \colon = \frac{1}{n}\sum_{i=1}^n \left(y_i - f(\*x_i)\right)^2,
	\]
	\[
	\Delta_n\coloneqq \#E_{f_0}\left[ \ell_n(f_{\operatorname{M}})-\inf_{f\in \^F}\ell_n(f)\right],
	\]
	 where $\#E_{f_0}$ is the expectation taken with respect to the samples generated from the regression model (\ref{eq:reg}). Define
	 \[
	 \operatorname{dis}(f_{\operatorname{M}},f_0) \coloneqq
	 \#E_{f_0}\left[\left(f_{\operatorname{M}}(\*x) - f_{0}(\*x)\right)^2\right].
	 \]
	 Then, we have
	 \[
	 \operatorname{dis}(f_{\operatorname{M}},f_0) \leq \^O\left(\Delta_n + \inf_{f\in \^F}\operatorname{dis}(f,f_0)+ \frac{sl^2 \ln \left(w n\right)}{n}\right).
	 \]
\end{corollary}
This corollary is indeed a direct application of the general statics generalization inequality in~\cite{lu2020deep,yarotsky2017error}. As we can see, the generalization bound depends on three parts, intuitively described  as $\varepsilon_1+\varepsilon_2 + \varepsilon_3$, where $\varepsilon_1$ is the gap from the obtained training loss to the global minimal one, $\varepsilon_2$ is the approximation error, and $\varepsilon_3$ is the covering number. Notably, Theorem \ref{thm:depth} provides a way to reduce $\varepsilon_1$, and Theorem~\ref{thm:covering} ensures that small $\varepsilon_2$ unnecessarily results in large $\varepsilon_3$.
\par
For nonparametric regression with square loss, when the target function $f_0$ is $\beta$-smooth, it is well-known that the statistically optimal estimation rate in terms of data size is $n^{-\frac{2\beta}{2\beta + d_x}}$~\cite{gine2016mathematical}, also called as \emph{minimax estimation rate}. Owning the minimax estimation rate means that the estimator performs the best in the worst case. Interestingly, when the training data is fitted exactly, MCN also owns this property.
\begin{theorem}[Minimax Estimation Rate]\label{thm:gene}
Suppose that the density $p(\cdot)$ over some compact set $\^C$ satisfies
	\[0 < p_{\min} \leq p(\*x) \leq p_{\max}, \quad \forall \*x \in \^C.\]
Assume that the target function $f_0$ is $\beta$-smooth and let $\ell(\cdot)$ in (\ref{eq:tr_loss}) be the square loss. Denote the final output of our model as $f_{\bm{\theta}}(\*x_i)$, where $\bm{\theta}$ is the learnable parameters of MCN. If $f_{\bm{\theta}}(\*x_i) = y_i$ for $i =1,\cdots,n$, then for any data sample $\*x \in \mathbb{R}^{d_x}$ located in the support set of $p$, the output of MCN satisfies the following with high probability:
	\[
	\mathbb{E}_{\^S^n}\left[	\mathbb{E}_{\varepsilon}\qty[\|f_{\bm{\theta}}\left(\*x\right)-f_0\left(\*x\right)\|^{2} \mid \^S^n]\right] \leq C n^{-\frac{2\beta}{2\beta+d_x}},
	\]
	where $\^S^n = \qty{\qty(\*x_i,y_i))}_{i=1}^n$ and $C > 0$ is a number depends only on the numerical range of the outputs of MCN.
\end{theorem}
In general, the above theorem confirms the phenomenon that over-parameterized DNNs may not necessarily cause over-fitting~\cite{pmlr-v89-belkin19a,belkin2018overfitting}. For Theorem~\ref{thm:gene} to hold, the training error has to be reduced to zero. This can actually be accomplished by using the techniques in~\cite{gasca2000polynomial} to link together Theorem~\ref{thm:depth} and Theorem~\ref{thm:approximation}, as will be shown in next subsection.

\textbf{Remark 2}: One may worry about the ``exact fitting'' assumption may not be satisfied since the noise belongs to an unbounded distribution.  
The derivatives or weights of DNN may diverge to infinity as $n \to \infty$.
However, this may not be a problem and exact fitting can easily happen under mild condition.
On the one hand, Gaussian distribution has an exponential decay tail. 
Thus, we can approximately treat it as bounded.
On the other hand, some recent results~\cite{arora2019fine,du2018gradient1,allen2018learning,liang2020multiple,ma2019generalization} show that the DNNs, having universal approximation ability, can easily fit the Gaussian noise without any weight diverging. Even more, exact fitting can happen near the initial state of DNN as long as it is wide or deep enough; the depth or width is in the polynomial order of $n$. 
For MCN, we already prove its approximation ability in Theorem \ref{thm:approximation}. 
Following the same road-map, we can conclude that its parameters do not diverge in the exact fitting case.\\
\textbf{Remark 3}: We remark that the estimator $f_{\bm{\theta}}$ does not belong to the $\beta$-smooth function class (its smoothness depends on the architecture and activation function). 
In conclusion, even though $f_{\bm{\theta}}$ is not $\beta$-smooth and fits the data exactly, it attains optimal excess loss rates. We refer the readers to~\cite{rakhlin2017empirical} for further discussion of optimal rates in non-parametric estimation and statistical learning.
\\
\textbf{Remark 4}: For any $\*x_i \in \^S^n$, \
\[
\mathbb{E}_{\epsilon_i}\left[\|f_{\bm{\theta}}\left(\*x_i\right)-f_0\left(\*x_i\right)\|^{2} \mid \^S^n\right] = 1.
\]
However,
\[
\mathbb{E}_{\^S^n}\left[	\mathbb{E}_{\varepsilon}\qty[\|f_{\bm{\theta}}\left(\*x_i\right)-f_0\left(\*x_i\right)\|^{2} \mid \^S^n]\right] \to 0, \text{ as } n \to \infty,
\]
due to the measure of a specific point is $0$.
\begin{figure*}[!htp]
	\centering 
	\includegraphics[width=1.0\textwidth]{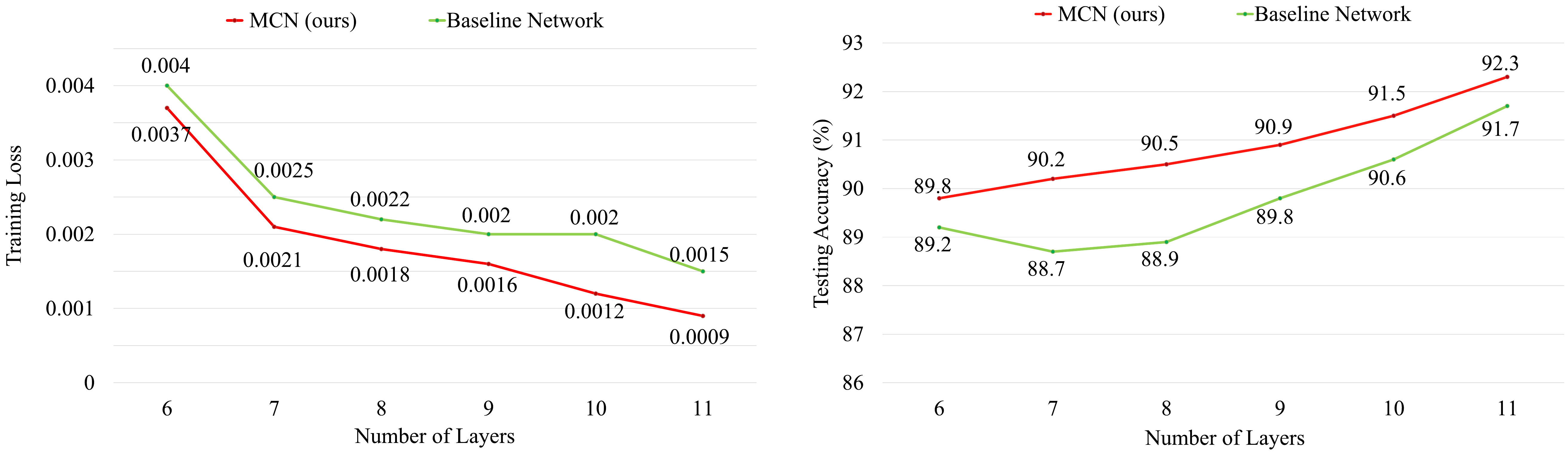}
	\caption{Left: Training loss of our MCN~(red) and baseline network~(green) with various number of layers. Right: Testing accuracy.}
	\label{fig:EXP-LM-CVNN}
\end{figure*}

\subsection{No Bad Local Minima}\label{sec:nobadlocal}
As aforementioned, under mild technical conditions, the training error produced by MCN can be arbitrarily small when the network is deep enough.
\begin{corollary}[Optimal Training Error]\label{corollary:zerotraining} Suppose that the loss function $\ell(\cdot)$ is differentiable and strongly convex. Denote by $\bm{\theta}_{l}$ any local minimum of an $l$-layer MCN. For any $\epsilon>0$, there exists a $D\in \#N$ such that $L(\bm{\theta}_{l}) \leq \epsilon$ holds for any $l>D$.
\end{corollary}
The ``no bad local minima'' property of MCN replies on its special network design, and is unnecessarily true for the other DNNs. In the following, we shall introduce two ways to refine an existing DNN that is possibly poorly designed. The first one is straightforward and simply to treat the output of an existing DNN as the input $\*x_0$ to MCN, and the parameters of the existing network are not involved in re-training. In this case, it is easy to obtain the following result:
\begin{corollary}[Partial Training]\label{corollary:part}
For fixed injection $\Psi(\cdot)$  and an existing $l_0$-layer DNN with output $\*h_0$, construct an $l$-layer MCN with $\gamma(\cdot)$ being element-wisely exponential and input $\*x_0=\*h_0$. If $\*h_0$ is an injective function w.r.t. the input $\*x$ and the loss $\ell(\cdot)$ is differentiable and strongly convex, then for any $\epsilon> 0$, there exists a large enough $D \in \#N$, such that $L(\bm{\theta}_{l}) \leq \epsilon$ holds for any local minimum $\bm{\theta}_{l}$ with $l \geq D$.
\end{corollary}
In above corollary, the existing DNN is assumed to be fixed and MCN is simply applied to its output. Actually, it is also feasible to re-train all the parameters, including the parameters of both the existing network and the appended MCN blocks.
\begin{theorem}[Full Training]\label{thm:full}
For fixed injection $\Psi(\cdot)$ and an existing $l_0$-layer DNN with output $\*h_0$, append an $l$-layer MCN at its end with $\gamma(\cdot)$ being element-wisely exponential, resulting in a new model $\*h_l$. Suppose that the loss $\ell(\cdot)$ is differentiable and strongly convex, and there exist parameters that make $\*h_0$ be injective. Then, for any $\epsilon> 0$, there exists a large enough $D \in \#N$ such that
	\[
	\frac{1}{n}\sum_{i=1}^n\ell(\Phi\left(\*h_l(\*x_i)\right), \*y_i) \leq \epsilon
	\]
	holds at any local minimum $\*h_l$ with $l\geq D$.
\end{theorem}
One may have noticed that monotonic decreasing property in Theorem \ref{thm:depth} is not enough to guarantee global minimal training loss. In fact, as aforementioned, Theorem~\ref{thm:approximation} also plays an important role in gaining the above results, and we need use the techniques in~\cite{gasca2000polynomial} to link Theorem~\ref{thm:depth} and Theorem~\ref{thm:approximation} together.

Remarkably, the above results illustrate that MCN is not just an approach for seeking the global optimal solution to certain optimization problems, but instead a powerful tool for helping seek better solutions to the primary task behind the optimization problems.
\subsection{Discussions}
There is another interpretation for why MCN can eliminate bad local minimum. When adopting the square loss, we find that the loss in (\ref{eq:tr_loss}) at the local minimum equals to a projection residual obtained by projecting the training data onto a subspace. The subspace is expanded by parameters in the concatenation linear part $\^L_k(\cdot)$ for $k = 1,\cdots,l$, which means that the subspace is larger when more independent parameters are contained in the linear branch $\^L_k(\cdot)$. On the other hand, large space often brings small projection residual. Please see Section \ref{sec:connectLR} in the supplementary material for more details.
\par
To summarize, this section establishes a collection of theorems to cope with the problems of bad local minima and generalization issue. More precisely, first, Theorem~\ref{thm:depth} and Corollary~\ref{corollary:zerotraining} reveal the ``no bad local minima'' property of MCN, and Corollary~\ref{corollary:part} and Theorem~\ref{thm:full} extend this property to the other DNNs. Second, Theorem \ref{thm:approximation} shows the approximation ability of MCN, illustrating that MCN can obtain the same approximation error by using parameters much less than deep ReLU. The number of required parameters is far smaller than the network size, which implies that MCN allows to use some prevalent sparse patterns such as CNN structure and pruning tricks. The sparsity of network connections further leads to a small covering number for MCN in Theorem \ref{thm:covering}. Based on this, finally, we provide the generalization bound for MCN in Corollary \ref{corollary:risk}.
\section{Experiments}
\begin{table*}[h]
	\renewcommand{\arraystretch}{1.0}
	\begin{centering}
		\begin{threeparttable}[]
			\caption{The training error (Err.) and testing accuracy (Acc.) of different models on the CIFAR-10 dataset. We denote by C the added two convolutional layers and M the appended MCN blocks.}\label{tab:exp2}
			\tabcolsep=5.5pt
			\begin{minipage}{12cm}
				\begin{tabular}{@{}l|ccccc|cccccc}
					\toprule[0.45pt] 
					Models         &VGG19 &VGG19+ &VGG19+ & VGG19+ & VGG19+ & Res18 & Res18+ & Res18+ & Res18+ & Res18+ \\
					&  &C(full)&C(part) &M(full)& M(part) &   &C(full)&C(part) & M(full)&M(part) \\\hline
					Err.     & 0.0016   & 0.0013 & 0.0015 &\textbf{0.0010} & 0.0011 &0.0013  & 0.0011 & 0.0012  &\textbf{0.0009}& \textbf{0.0009}&  \\
					Acc.    & 92.0\%  & 92.4\% &  92.1\% &\textbf{92.8\%}  & 92.6\% &92.7\% & 93.5\% & 93.1\%  &93.7\%& \textbf{93.8}\% \\
					\bottomrule[0.4pt]
				\end{tabular}
			\end{minipage}
		\end{threeparttable}
	\end{centering}
\end{table*}
\begin{table*}[!htb]
	\renewcommand{\arraystretch}{1.0}
	\begin{centering}
		\begin{threeparttable}[]
			\caption{The training error (Err.) and testing accuracy (Acc.) of different models on the CIFAR-100 dataset. We denote by C the added two convolutional layers and M the appended MCN blocks.}\label{tab:exp3}
			\tabcolsep=1.9pt
			\begin{minipage}{12cm}
				\begin{tabular}{@{}l|ccccc|cccccc}
					\toprule[0.45pt] 
					Models         &Res18 &Res18+ &Res18+ & Res18+ & Res18+ & ResNeXt29 & ResNeXt29+ & ResNeXt29+ & ResNeXt29+ & ResNeXt29+ \\
					&  &C(full)&C(part) &M(full)& M(part) &   &C(full)&C(part) & M(full)&M(part) \\\hline
					Err.     & 0.0020   & 0.0014 & 0.0015 & \textbf{0.0009} & \textbf{0.0009} &0.0056  & 0.0051 & 0.0054  & \textbf{0.0008} & {0.0011}&  \\
					Acc.    & 76.15\%  & 76.58\% &  76.48\% & \textbf{76.95\%} & {76.87\%} &80.71\% & 80.78\%  & 80.69\%  & \textbf{82.31\%} & {81.41}\% \\
					\bottomrule[0.4pt]
				\end{tabular}
			\end{minipage}
		\end{threeparttable}
	\end{centering}
\end{table*}
\subsection{Theorems Verification}
We conduct experiments on the commonly used CIFAR-10 dataset, with the purpose of validating our theorems as well as the effectiveness of MCN. We first construct a baseline network with 6 weighted layers, including five convolutional layers and one fully-connected layer. Then we add convolutional layers to make the network deeper.
It contains five max pooling in total. For our MCN, we replace the convolutional layers after the third max pooling layer with our MCN block. To make a fair comparison, both networks have the same number of layers and parameters, and so for the random seed and learning rate. Also, batch normalization and ReLU are adopted by both networks. For detailed experimental settings and model configurations, please refer to the supplementary material.
\par
Figure~\ref{fig:EXP-LM-CVNN} shows the training loss and testing accuracy with different number of layers. According to the red line in the left part of Figure~\ref{fig:EXP-LM-CVNN}, the training loss of our MCN monotonically decreases with the increase of depth. This is consistent with our Theorems~\ref{thm:depth} and~\ref{thm:approximation}. From the red line in the right part of Figure~\ref{fig:EXP-LM-CVNN}, we can see that deeper MCN can achieve better testing accuracy, which demonstrates the generalization performance of MCN and confirms our Corollary \ref{corollary:risk} and Theorem~\ref{thm:gene}. In addition, according to the green line in the right part of Figure~\ref{fig:EXP-LM-CVNN}, the testing accuracy of the baseline network does not monotonically increase as the network goes deeper. Therefore, the ``no bad local minima'' property should be a primary cause of the nice performance of MCN. In summary, compared with the baseline network, our MCN has much lower training loss as well as higher testing accuracy, revealing the superiority of MCN.
\subsection{Appending MCN}
To validate the merits of Corollary \ref{corollary:part} and Theorem~\ref{thm:full}, we add two MCN blocks to VGG19~\cite{simonyan2014very} and ResNet18~\cite{he2016deep} as the treatment group. The original two architectures, VGG19 and ResNet18, are regarded as the first control group. To make a comparison, we also add two traditional convolutional layers to VGG19 and ResNet18, considered as the second control group. For the treatment group and the second control group, we consider two ways to train the appended VGG19 and ResNet18 (short as Res18). The first one is partial training which treats VGG19 and Res18 as the feature extractors whose parameters are not involved during training. The second one is full training which considers the appended networks as new models and train them from scratch.

Table~\ref{tab:exp2} shows the comparison results among all the three groups, in terms of both training loss and testing accuracy. 
As we can see, the plugging of traditional convolution layers can decrease the training loss, however, the appending of MCN has more amount of improvement, which, again, show the benefits of the ``no bad local minima" property.
Interestingly, full training and partial training share comparable performance when appending MCN but not for convolution layers. Hence, both Corollary \ref{corollary:part} and Theorem \ref{thm:full} are practical theories.
Moreover, our MCN outperforms distinctly all the competing methods; this, again, confirms the superiority of our MCN architecture.

\subsection{Additional Experiments}
To better demonstrate the representation ability of our MCN block, we further conduct some additional experiments on the more complex dataset CIFAR-100, and make comparisons with the SOTA of ResNeXt~\cite{ResNext} (a more powerful network architecture). 
\par
Similar to the previous part, the original two architectures, Res18 and ResNeXt29, are regarded as the first control group. 
As for the second control group, we still add two traditional convolutional layers to Res18 and ResNeXt29.
Besides, we append two MCN blocks to the end of both Res18 and ResNeXt29 as the treatment group.
\par
For the treatment group and the second control group, the two ways to train the appended Res18 and ResNeXt29 remain the same as previous experiment. One is partial training which treats Res18 and ResNeXt29 as the feature extractors, while the other is full training which considers the appended DNNs as new models and train them from scratch.
\par
We present the results of the partial training (i.e., fixing Res18 and ResNeXt29 when appending MCN blocks) in Table~\ref{tab:exp3}. It can be seen that, even in the case of handling complex data, our MCN achieves superior results. The treatment groups under two different training methods both outperform the control groups, which is consistent with Table~\ref{tab:exp2}.
Moreover, by comparing Table~\ref{tab:exp2} with Table~\ref{tab:exp3}, our MCN blocks have greatly improved the performance when handling more complex data. Please note that ordinarily appending CNNs cannot ensure the monotonicity of Err. and Acc. This phenomenon not only verifies Corollary \ref{corollary:part} and Theorem \ref{thm:full} again, but also shows that our MCN has a stronger representation ability than general linear structure, which corresponds to Theorem~\ref{thm:approximation}.

\section{Conclusion}
In this paper, we propose a novel multi-layer DNN structure termed MCN, which can approximate some class of continuous functions arbitrarily well even with highly sparse connection.
We prove that the global minima of an $l$-layer MCN may be outperformed, at least can be attained, by simply increasing the network depth.
More importantly, MCN could be easily appended to any of the many existing DNN and the augmented DNN will share the same property of MCN.
Finally, we analyze the generalization ability of MCN and reveal that depth is more important than width for generalization;
this supports the mechanism of deep learning.
In summary, this study does take a step towards the ultimate goal of deep learning theory---to understand why DNNs can work well in a wide variety of applications.

\section*{Acknowledgments}
This work is supported in part by New Generation AI Major Project of Ministry of Science and Technology of China (grant no 2018AAA0102501), in part by NSF China (grant no.s 61625301 and 61731018), in part by Major Scientific Research Project of Zhejiang Lab (grant no.s 2019KB0AC01 and 2019KB0AB02), in part by Fundamental Research Funds of Shandong University, in part by Beijing Academy of Artificial Intelligence, in part by Qualcomm, and in part by SenseTime Research Fund.

\bibliography{icml_refs}
\bibliographystyle{icml2020}
\newpage
\onecolumn
\appendix
\section{Appendix}
\subsection{Experimental Settings and Model Configuration}
For the baseline network, it is a reduced version of the VGG network.
We adopt the similar structure as that in \footnote{https://github.com/kuangliu/pytorch-cifar/blob/master/models/vgg.py}, where the last layer is a fully-connected layer and all other weighted layers are convolutional layers.
It contains five max pooling  in total.
For our MCN, we replace the convolutional layers after the third max pooling layer with our MCN block introduced in the right part of Figure~\ref{fig:LM-CVNN}.
For each MCN block, the upper convolutional operation has $128$ $3\times 3$ kernels, and other three operations has $256$ $3\times 3$ kernels
The model configuration of MCN is presented in Table~\ref{tab:conf}.
For fair comparison and for all models with different layers, we set the learning rate to $1\times 10^{-4}$ and total number of epochs to $250$, respectively.
\begin{table}[h]
	\scriptsize
	\renewcommand{\arraystretch}{1.3}
	\setlength{\tabcolsep}{2pt}
	\centering
	\caption{Model configuration of MCN.}
	\begin{tabular}{|c|c|c|c|c|c|}
		\hline
		\multicolumn{6}{|c|}{MCN Configuration}                                                                                                                                                   \\ \hline
		6 weight layers               & 7 weight layers               & 8 weight layers               & 9 weight layers               & 10 weight layers              & 11 weight layers              \\ \hline
		\multicolumn{6}{|c|}{Input (32 $\times$ 32 RGB image)}                                                                                                                           \\ \hline
		3 $\times$ 3 conv. 64 BN ReLU & 3 $\times$ 3 conv. 64 BN ReLU & 3 $\times$ 3 conv. 64 BN ReLU & 3 $\times$ 3 conv. 64 BN ReLU & 3 $\times$ 3 conv. 64 BN ReLU & 3 $\times$ 3 conv. 64 BN ReLU \\ \hline
		&                               &                               &                               &                               & 3 $\times$ 3 conv. 64 BN ReLU \\ \hline
		\multicolumn{6}{|c|}{Max pooling}                                                                                                                                                             \\ \hline
		3 $\times$ 3 conv. 128 BN ReLU & 3 $\times$ 3 conv. 128 BN ReLU & 3 $\times$ 3 conv. 128 BN ReLU & 3 $\times$ 3 conv. 128 BN ReLU & 3 $\times$ 3 conv. 128 BN ReLU & 3 $\times$ 3 conv. 128 BN ReLU \\ \hline
		&                               &                               &                               & 3 $\times$ 3 conv. 128 BN ReLU & 3 $\times$ 3 conv. 128 BN ReLU \\ \hline
		\multicolumn{6}{|c|}{Max pooling}                                                                                                                                                             \\ \hline
		3 $\times$ 3 conv. 256 BN ReLU & 3 $\times$ 3 conv. 256 BN ReLU & 3 $\times$ 3 conv. 256 BN ReLU & 3 $\times$ 3 conv. 256 BN ReLU & 3 $\times$ 3 conv. 256 BN ReLU & 3 $\times$ 3 conv. 256 BN ReLU \\ \hline
		&                               &                               & 3 $\times$ 3 conv. 256 BN ReLU & 3 $\times$ 3 conv. 256 BN ReLU & 3 $\times$ 3 conv. 64 BN ReLU \\ \hline
		\multicolumn{6}{|c|}{Max pooling}                                                                                                                                                             \\ \hline
		MCN Block                 & MCN Block                 & MCN Block                 & MCN Block                 & MCN Block                 & MCN Block                 \\ \hline
		&                               & MCN Block                 & MCN Block                 & MCN Block                 & MCN Block                 \\ \hline
		\multicolumn{6}{|c|}{Max pooling}                                                                                                                                                             \\ \hline
		MCN Block                 & MCN Block                 & MCN Block                 & MCN Block                 & MCN Block                 & MCN Block                 \\ \hline
		& MCN Block                 & MCN Block                 & MCN Block                 & MCN Block                 & MCN Block                 \\ \hline
		\multicolumn{6}{|c|}{Max pooling}                                                                                                                                                             \\ \hline
		\multicolumn{6}{|c|}{FC-10}                                                                                                                                                                   \\ \hline
		\multicolumn{6}{|c|}{Soft-max}                                                                                                                                                                \\ \hline
	\end{tabular}
	\label{tab:conf}
\end{table}

It is worth noting that our MCN contains only a small amount of parameters. Specifically, in our configuration, each MCN block has only $3\times 3 \times (256\times 3+128) = 8,064$ parameters. And MCN uses only a single fully-connected layer, which leads to $512\times 10 = 5,120$ parameters. So, even for an MCN with $11$ layers, the total number of parameters is less than $5\times 10^4$.


\subsection{Proof of Theorem \ref{thm:depth}}
\begin{proof}
	Since $\^L_{k+1}, \^W_{k+1}, \^{\tilde{A}}_{k+1}$ and $\^A_{k+1}$ are linear operators, ignoring the biases, we simplify MCN as:
	\[
	\*x_{k+1,i} = \left[\*L_{k+1}\*x_{k,i};~\gamma\left(\*{\tilde{A}}_{k+1}\*x_i\right) +\max\left\{\*W_{k+1}\*x_{k,i},  \sigma\left(\*A_{k+1}\*x_i\right)\right\}\right],
	\]
	where $\{\*L_{k+1}, \*W_{k+1},  \*{\tilde{A}}_{k+1}, \*A_{k+1}\} \in \bm{\theta}_{k+1}$, $\bm{\theta}_{k+1}$ is a local minimum of the loss $L$ and $\*x_i \in \{\*x_i\}_{i=1}^n$ is an arbitrary training sample. For convenience, in this proof, we assume $\*x_i \neq \*x_j$ when $i\neq j$ for any $\*x_i$ and $\*x_j \in \{\*x_i\}_{i=1}^n$ and $\sigma(\cdot) = \sigma_{k}(\cdot),~\forall k \in [l]$.
	\par
	Let $\ell_{\Psi}(\*x_{k+1,i}):=\ell(\Psi(\*x_{k+1,i}), \*y_i)$ and $\nabla \ell_{\Psi}(\*x_{k+1,i})$ be the gradient $\nabla \ell_{\Psi}$ evaluated at $\*x_{k+1,i}$. We can have the following claim.
	\begin{claim}\label{clm:zeroMapping}
		With the same setting in Theorem \ref{thm:depth}, for any $\*u\in \#R^{d_x}$ with $\norm{\*u}_2 = 1$ and $t\in \#N$ we have:
		\[
	 \sum_{i=1}^n c_{i,j,t} \left(\nabla \ell_{\Psi}(\*x_{k+1,i})\right)_j \left(\*u_j\top\*x_i\right)^t = 0,\quad \forall j \in [d_{\^L}+1, d_{k+1}],\quad \forall i \in [n],
		\]
		where
		\[
		c_{i,j,t} \coloneqq \gamma^{(t)}\left(\left(\*{\tilde{A}}\*x_i\right)_j\right).
		\]
	\end{claim}
	\begin{proof}
		Let $\*\Lambda_{k+1,i} \in \mathbb{R}^{(d_{k+1} - d_{\^L}) \times (d_{k+1} - d_{\^L}) }$ represents a diagonal matrix with diagonal elements corresponding to the maximum pattern of the data point $\*x_i$ at the $(k+1)$-th layer as:
		\[
		\left(\*\Lambda_{k+1,i}\right)_{(j,j)}:=
		\begin{cases}
		1,& \quad \text{if}\quad \*x_{k,i}^\top (\*W_{k+1}^\top)_j  \leq \sigma\left( \*x_{i}^\top\left(\*A_{k+1}^\top\right)_j\right);\\
		0,& \quad \text{otherwise},
		\end{cases}
		\]
		where $(\*A_{k+1}^\top)_j$ is the $j$-th column of the matrix $\*A_{k+1}^\top$, i.e., the $j$-th row of $\*A_{k+1}$.
		We also define the  complement of the matrix  $\*\Lambda_{k+1,i}$:
		\[\overline{\*\Lambda}_{k+1,i} = \*I - \*\Lambda_{k+1,i},\]
		where $\*I \in \mathbb{R}^{(d_{k+1} - d_{\^L})}$ is the identity matrix. Without ambiguity, we omit the subscription $(k+1)$ for $\{\*L_{k+1}, \*W_{k+1},  \*{\tilde{A}}_{k+1}, \*A_{k+1}, \overline{\*\Lambda}_{k+1,i}, \*\Lambda_{k+1,i}\}$ and rewrite $\*x_{k+1}$ as:
		\[
		\*x_{k+1,i} = \left[\*L\*x_{k,i};~ \gamma\left(\*{\tilde{A}}\*x_i\right) +\overline{\*\Lambda}_{i} \*W\*x_{k,i}+ \*\Lambda_{i}\left( \sigma\left(\*A\*x_i\right)\right)\right].
		\]
		By perturbing parameters, we can define a new output:
		\[
		\*x'_{k+1,i} = \left[\*L\*x_{k,i};~ \gamma\left(\left(\*{\tilde{A}}+ \*{\Delta A}\right)\*x_i\right) +\overline{\*\Lambda}_{i}'  \*W\*x_{k,i}+ \*\Lambda_{i}' \left( \sigma\left(\*A\*x_i\right)\right)\right].
		\]
		In general, due to the perturbation, the maximum pattern $\*\Lambda_{i}$ will change. However, if the perturbation is small enough, i.e., $\|\*{\Delta A}\|$ is sufficiently small, we have $\*\Lambda_{i} = \*\Lambda_{i}'$. Then, we have:
		\[
		\*d_i := x_{k+1,i} - x_{k+1,i}' = \left[\*0;~ \gamma\left(\left(\*{\tilde{A}}+ \*{\Delta A}\right)\*x_i\right)-
		\gamma\left(\*{\tilde{A}}\*x_i\right) \right].
		\]
		For any $j \in [d_{\^L}, d_{k+1}]$, we let:
		\[
		e_{i,j}:=\*x_i^\top\*{\Delta a}_j,
		\]
		where $\*{\Delta a}_j =\*{\Delta A}_{(j,:)}$ is the $j$-th row of the matrix $\*{\Delta A}$. Then by the Taylor expansion of the function $\gamma(\cdot)$ at $\left(\*{\tilde{A}}\*x_i\right)_j$ for all $i,j$, we have
		\[
		\*d_{i,j} = \sum_{q=1}^\infty \frac{\gamma^{(q)}\left(\left(\*{\tilde{A}}\*x_i\right)_j\right)}{q!}	e_{i,j}^q.
		\] 
		\par
		Let $\bm{\tilde{\theta}}_{k+1} = \{\bm{\theta}_{k+1} \setminus \*A, \*A+ \*{\Delta A}\}$. Since $\bm{\tilde{\theta}}_{k+1}$ is a local minimum, we have that, for any sufficiently small $\*{\Delta A}$, we have:
		\[
		\begin{aligned}
		&n \left(L(\bm{\theta}_{k+1}) - L(\bm{\tilde{\theta}}_{k+1})\right) = \sum_{i=1}^n\ell(\Psi(\*x_{k+1,i}), \*y_i) - \sum_{i=1}^n \ell(\Psi(\*x'_{k+1,i}), \*y_i)
		=  \sum_{i=1}^n \left(\ell_{\Psi}(\*x_{k+1,i}) - \ell_{\Psi}(\*x'_{k+1,i})\right)\\
		\overset{(a)}{=}&\sum_{i=1}^n \left(\left( \nabla \ell_{\Psi}(\*x_{k+1,i})^\top \*d_i + \^O(\|\*d_i\|^2)\right)\right)
		\overset{(b)}{=} \sum_{j=d_{\^L}+1}^{d_{k+1}} \sum_{i=1}^n\left( \left(\nabla \ell_{\Psi}(\*x_{k+1,i})\right)_j \*d_{i,j}\right) +\^O\left(\|\*{\Delta A}\|^2\right)\\
		=& \sum_{j=d_{\^L}+1}^{d_{k+1}} \left(\sum_{q=1}^\infty
		\frac{z_{j,q}}{q!}
		\right) + \^O\left(\|\*{\Delta A}\|^2\right)
		\leq  0,
		\end{aligned}
		\]
		where
		\[
		z_{j,q} \coloneqq \left( \sum_{i=1}^n c_{i, j,q}\left(\nabla \ell_{\Psi}(\*x_{k+1,i})\right)_j e^q_{i,j}\right),\quad
		c_{i, j,q} \coloneqq \gamma^{(q)}\left(\left(\*{\tilde{A}}\*x_i\right)_j\right),
		\]
		and $(a)$ comes from the definition of differentiability for multivariable function and $(b)$ is due to the boundness of the first derivative of $\gamma(\cdot)$.
		Since the sum is the dominant term, then we can have:
		\[
		\sum_{j=d_{\^L}+1}^{d_{k+1}} \left(\sum_{q=1}^n
		\frac{z_{j,q}}{q!} 
		\right) \leq 0.
		\]
		Due to this inequality holds for any sufficient small $\*{\Delta a}_j$, we can conclude that
		\[
		\sum_{q=1}^\infty
		\frac{z_{j,q}}{q!}  = 0,\quad \forall j.
		\]
		By setting $\*{\Delta a}_j = \epsilon_j \*u_j$ such that $\epsilon_j>0$ and $\|\*u_j\| = 1$, we have:
		\[
		\sum_{q=1}^\infty \frac{\epsilon_j^q}{q!} \sum_{i=1}^n c_{i,j,q} \left(\nabla \ell_{\Psi}(\*x_{k+1,i})\right)_j \left(\*u_j^\top\*x_i\right)^q =0
		 ,\quad \forall j.
		\]
		Now, we set 
		\[
		\eta_q = \left( \sum_{i=1}^n c_{i, j,q} \left(\nabla \ell_{\Psi}(\*x_{k+1,i})\right)_j \left(\*u_j^\top\*x_i\right)^q\right).
		\]
		Divide the $\epsilon_j$ on both side, we can get:
		\[
		\eta_1 + 	\sum_{q=2}^\infty\frac{\epsilon_j^{q-1}}{q!} = 0,\quad \forall j.
		\]
		Note that 
		\[
		\sum_{q=2}^\infty\frac{\epsilon_j^{q-1}}{q!} \to 0, \quad \epsilon_j 0.
		\]
		Then, we get $\eta_1 = 0$. We can multiplying $p!/\epsilon_j^q$ on both sides and prove by induction that
		\[
		\eta_q = 0, \text{ for }\quad q=1,\cdots.
		\]
		We finish the proof of this claim.
	\end{proof}
	Given any $i \in \{1,\cdots,n\}$, consider the case:
	\[
	\left(\nabla \ell_{\Psi}(\*x_{k+1,i})\right)_j = 0, \quad \forall j \in [d_{\^L}, d_{k+1}].
	\]
	We can rewrite the above equation as:
	\[
	\Psi^\top \left(\nabla \ell\left(\Psi(\*x_{k+1,i}), y_i\right)\right) = 
	\begin{bmatrix}
	*\\ \*0
	\end{bmatrix},
	\]
	where $\nabla \ell(\cdot)$ is the gradient of $\nabla \ell$, e.g.,   $\nabla \ell\left(\Psi(\*x_{k+1,i}), y_i\right) = \Psi(\*x_{k+1,i})-y_i$ for squared loss or $\nabla \ell\left(\Psi(\*x_{k+1,i}), y_i\right) = \eta(\Psi(\*x_{k+1,i}))-y_i$, where $\eta(\cdot)$ is the softmax function for cross entropy loss, $\*0 \in \mathbb{R}^{d_{k+1} - d_{\^L}}$ and $*$ is an arbitrary vector in $\mathbb{R}^{d_{\^L}}$. Since $\Psi(\cdot)$ is surjection
	\footnote{Let $\Psi(\*x)=
		\begin{bmatrix}
		\Psi_1(\*x)&\Psi_2(\*x)
		\end{bmatrix}$. Actually, it needs $\Psi_2(\*x)$ to be surjective here. However, the entries' order of MCN's each layer can be arbitrary and $\Psi(\cdot)$ is fixed. Hence, we can always change the order of entries of $\*x$ to let $\Psi_2(\*x)$ be surjective without changing the values of learnable parameters. 	
}
	, we can conclude that:
	\[
	\ell\left(\Psi(\*x_{k+1,i}), y_i\right) = 0, \quad \forall i \in [n],
	\]
	which completes this proof. Therefore, for the sake of simplicity, we exclude this all zero case and assume $\left(\nabla \ell_{\Psi}(\*x_{k+1,i})\right)_{d_{\^L}+1}\neq 0$ in the following proof.
	\par
	Given $\bm{\theta}_{k+1}$ is a local minimum of $L$, by the convexity of the function $\ell_{\Psi}(\*x_{k+1,i})$, for any $\bm{\theta}'_{l}$, we have:
	\[
	n\left( L(\bm{\theta}'_{l}) - L(\bm{\theta}_{k+1})\right) \geq \sum_{i=1}^n \nabla \ell_{\Psi}(\*x_{k+1,i})^\top \left(\*x'_{k,i} - \*x_{k+1,i}\right) = \underbrace{\sum_{j=1}^{d_{k+1}}\sum_{i=1}^n \left(\nabla \ell_{\Psi}(\*x_{k+1,i})\right)_j \left(\*x'_{k,i} - \*x_{k+1,i}\right)_j}_{\text{Lower Bound } L_B}.
	\]
	Denote by $\otimes$  the tensor product and let $\*x^{\otimes p}\coloneqq \*x \otimes \cdots \otimes \*x$. For a $p$-th order tensor $\*M \in \#R^{d\times \cdots \times d}$ and $p$ vectors $\{\*u_1,\cdots,\*u_p\}$,  let 
	\[
	\*M\left(\*u_1,\cdots,\*u_p\right) \coloneqq \sum_{1\leq i_1,\cdots,i_p\leq d} \*M_{i_1,\cdots,i_p}\*u_{1,i_1}\cdots\*u_{p,i_p}.
	\]
	It is known from~\cite{zhang2012best}, given $n, p>0$ and $\xi_i$ for $i = 1,\cdots,n$,
	\[
	\max_{\|\*u_1\|_2 = \|\*u_2\|_2 = \cdots = \|\*u_p\|_2 =1} \left(\sum_{i=1}^n \xi_i \*x_i^{\otimes p}\right) \left(\*u_1,\cdots,\*u_p\right) = 
		\max_{\|\*u\|_2} \left(\sum_{i=1}^n \xi_i 
	\left(\*u^\top \*x_i\right)^{p}\right).
	\]
	Hence, with this observation, together with the results in Claim \ref{clm:zeroMapping}, we get
	\[
	\sum_{i=1}^n c_{i, j,t} \left(\nabla \ell_{\Psi}(\*x_{k+1,i})\right)_j \operatorname{vec}(\*x_i^{\otimes t}) = 0,\quad \forall j \in [d_{\^L}+1, d_{k+1}],\quad \forall t \in [n]
	\]
	Before proceeding, we provide a result of the existence of a polynomial interpolation of the finite distinct $n$ points; interpolation of finite $n$ points.
	\begin{claim}[Polynomial Interpolation\cite{gasca2000polynomial}]\label{clm:interpolation}
		Let $\{\*x_i\}_{i=1}^n$ be distinct points in $\mathbb{R}^{d_x}$. For any $d_x$–dimensional continuous functions $f : \#R^{d_x}\rightarrow \#R$, consider the set $\*\Omega \coloneqq\{f(\*x_1),\cdots ,f(\*x_n)\}$.
		There exists a $r$-th order polynomial $q(\cdot):\#R^{d_x}\rightarrow \#R$ such that interpolate the points in the set $\*\Omega$, where the order $r\leq (n-1)$; namely, there exists the vectors $\{\*u_t \in \#R^{d^t_x}\}$ for $t=1,\cdots,r$ such that 
		\[
		f(\*x_i) = q(\*x_i) = \sum_{t=1}^{r} \*u_t^\top \operatorname{vec}\left(\*x_i^{\otimes t}\right), \quad \forall \*x_i \in \{\*x_i\}_{i=1}^n.
		\]
	\end{claim}
	By this claim, it is easy to conclude that the difference of two continuous functions $f_1(\cdot)$ and $f_2(\cdot)$ can also be interpolated; namely, there exists vectors $\{\*u^{(1)}_t \in \#R^{d^t_x}\}$ and $\{\*u^{(2)}_t \in \#R^{d^t_x}\}$:
	\[
	f_1(\*x_i) - f_2(\*x_i) = \sum_{t=1}^{r}{\*u^{(1)}_t}^\top \operatorname{vec}\left(\*x_i^{\otimes t}\right) -  \sum_{t=1}^{r}{\*u^{(2)}_t}^\top \operatorname{vec}\left(\*x_i^{\otimes t}\right) =   \sum_{t=1}^{r}\left(\*u^{(1)}_t - \*u^{(2)}_t\right)^\top \operatorname{vec}\left(\*x_i^{\otimes t}\right)
	\coloneqq \sum_{t=1}^{r} \*u_t^\top \operatorname{vec}\left(\*x_i^{\otimes t}\right).
	\]
	Note that when $\gamma(\cdot) = \exp(\cdot)$, we have $c_{i,j,t_1} = c_{i,j,t_2}$ when $t_1 \ne t_2$. Hence, we omit the subscript $t$.
	Notice that for any $j \in [d_{k+1}]$, $\left(\*x'_{k,i}\right)_j$ and $\left( \*x_{k+1,i}\right)_j$ are always continuous functions of $\*x_i$. Hence, for all $i$, there exists vectors $\{\*u_{t,j} \in \#R^{d^t_x}\}$ such that:
	\begin{equation}\label{eq:before_dL}
	\frac{1}{c_{i,j}}\left(\*x'_{k,i} - \*x_{k+1,i}\right)_j \frac{\left(\nabla \ell_{\Psi}(\*x_{k+1,i})\right)_j}{\left(\nabla \ell_{\Psi}(\*x_{k+1,i})\right)_{d_{\^L}+1}} =  \sum_{t=1}^{r} \*u_{t,j}^\top \operatorname{vec}\left(\*x_i^{\otimes t}\right), \quad \forall j \in [d_{\^L}],
	\end{equation}
	and
	\begin{equation}\label{eq:after_dL}
	\frac{1}{c_{i,j}}\left(\*x'_{k,i} - \*x_{k+1,i}\right)_j  = \sum_{t=1}^{r} \*u_{t,j}^\top \operatorname{vec}\left(\*x_i^{\otimes t}\right), \quad \forall j \in [d_{\^L}+1, d_{k+1}].
	\end{equation}
	If $\left(\nabla \ell_{\Psi}(\*x_{k+1,i})\right)_j = 0$ for some $j\in [d_{\^L}]$, then we can ignore this zero term in the lower bound $L_B$. Thus, for brevity, we assume that $\left(\nabla \ell_{\Psi}(\*x_{k+1,i})\right)_j \neq 0,~\forall j\in [d_{\^L}]$. Combing the Eq. (\ref{eq:before_dL}) and Eq. (\ref{eq:after_dL}), we have
	\[
	\begin{aligned}
	L_B = & \sum_{t=1}^{r} \sum_{j=1}^{d_{\^L}}\sum_{i=1}^n c_{i,j}\left(\nabla \ell_{\Psi}(\*x_{k+1,i})\right)_{d_{\^L}+1} \*u_{t,j}^\top \operatorname{vec}\left(\*x_i^{\otimes t}\right)
	+ \sum_{t=1}^{r} \sum_{j=d_{\^L}+1}^{d_{k+1}}\sum_{i=1}^n c_{i,j} \left(\nabla \ell_{\Psi}(\*x_{k+1,i})\right)_{j} \*u_{t,j}^\top \operatorname{vec}\left(\*x_i^{\otimes t}\right) \\
	=& \sum_{t=1}^{r} \sum_{j=1}^{d_{\^L}} \*u_{t,j}^\top\left(c_{i,j}\sum_{i=1}^n \left(\nabla \ell_{\Psi}(\*x_{k+1,i})\right)_{d_{\^L}+1} \operatorname{vec}\left(\*x_i^{\otimes t}\right)\right)
	+ \sum_{t=1}^{r} \sum_{j=d_{\^L}+1}^{d_{k+1}}\*u_{t,j}^\top\left(\sum_{i=1}^n \left(\nabla \ell_{\Psi}(\*x_{k+1,i})\right)_{j} \operatorname{vec}\left(\*x_i^{\otimes t}\right)\right) \\
	=& ~0,
	\end{aligned}
	\]
	where the last equality comes from the Claim \ref{clm:zeroMapping}. Therefore, when $\bm{\*\theta}_{k+1}$ is a local minimum of $L$, we have $L(\bm{\*\theta}'_l)\geq L(\bm{\*\theta}_{k+1})$ for any $\bm{\*\theta}'_l$.
	\par
	We now complete this proof.
\end{proof}

\subsection{Proof of Theorem \ref{thm:approximation}}
\begin{proof}
	We first provide several claims. Based on them, we can construct an MCN such that approximate the multivariate Fourier series will, which ensure the accurateness for approximation in the Sobolev space.
	\begin{claim}\label{clm:apprx_x2}
		The function $f(x) = x^2$ on the segment $[-1,1]$ can be approximated by an MCN of width $\^O(w)$ and depth $\^O(l)$ with the approximation error: 
		\[
		\epsilon = \^O(2^{-wl}).
		\]
		When $l$ is large enough, the number of non-zero parameters for this MCN is in the order of $\^O(w^2l)$.
	\end{claim}
	 \begin{proof}	
	 	We only consider the proof on the interval $[0, 1]$, the other half is the same.  Consider the $g : [0, 1] \to [0, 1]$,
	 	\[
	 	g_m(x)\coloneqq \max\{-\frac{x}{2},\frac{x}{2} - 2^{1-2m}\},
	 	\]
	 	and the nested function
	 	\[
	 	r_{m}(x)=g_m \circ g_{m-1} \circ \cdots \circ g_1(x).
	 	\]
	 	It is easy to see that $r_{m}(x)$ can be represented by the operator $\^M(\cdot)$ in MCN (see Eq. (\ref{Eq:LM-CVNN})). Hence, we can have one type of MCN $M_i(x)$ such that
	 	\[
	 	M_i\coloneqq 
	 	\begin{bmatrix}
	 	\sum_i\\g_{i+1}\\ g_{i+2}\\\vdots\\g_{i+m}
	 	\end{bmatrix},
		\text{
	 	 then
	 	}
	 	M_i\circ \left(
	 	\begin{bmatrix}
	 	r_i\\ r_{i+1}\\\vdots\\r_{i+m-1}
	 	\end{bmatrix}
	 	\right)
	 	=
	 	\begin{bmatrix}
	 	\sum_{i}^{i+m} r_i\\ r_{i+1}\\ r_{i+2}\\\vdots\\r_{i+m}.
	 	\end{bmatrix}.
	 	\]
	 	Now, we construct a three-layer MCN with $m$ units $[r_1,\cdots,r_m()]$ as the output.
	 	Note that 
	 	\[
	 	r_{m}(x)=\left\{\begin{array}{ll}
	 	{2^{-m}\left(\frac{2 k}{2^{m}} -x \right),} & {x \in\left[\frac{2 k}{2^m}, \frac{2 k+1}{2^m}\right], k=0,1, \cdots, 2^{m-1}-1}, \\\\
	 	{2^{-m}\left(x-\frac{2 k}{2^{m}}\right),} & {x \in\left[\frac{2 k-1}{2^{m}}, \frac{2 k}{2^{m}}\right], k=1,2, \cdots, 2^{m-1}}.
	 	\end{array}\right.
	 	\]
	 	is a ``sawtooth'' function. We now let $\^A_1(x)$ and $\^W_1(x)$ be:
	 	\[
	 	\^A_1(x) = 
	 	\begin{bmatrix}
	 	2^{-1}(-x)\\\vdots\\2^{-s}\left(\frac{2 k_s}{2^{s}} -x \right)
	 	\end{bmatrix},
	 	 \^W_1(x) = 
	 	\begin{bmatrix}
	 	 2^{-1}(x-1)\\\vdots\\2^{-s}\left(x-\frac{2 k'_s}{2^{s}}\right)
	 	\end{bmatrix},
	 	\quad s = 1,\cdots,m, \quad k'_s-1 = k_s = 0,1,\cdots,2^{s-1}-1.
	 	\]
	 	Hece, $\^A_1$ and $\^W_1$ map the input from $\#R \to \#R^{p}$, where $p = 2^{m}-1$ and the $(2^{s-1} + k_s)$-th entry of $\^M_1(x) = \max\{\^A_1(x), \^W_1(x)\} $ is the $k_s$-th ``tooth'' of $r_s(x)$ when $r_s(x)<0$. Let $\^W_2$ be the sign reversal operator and $\^A_2$ be the zero mapping, then we have
	 	\[
	 	\^M_2(x) = \max\{-\^M_1(x), 0\}=
	 	\left\{\begin{array}{ll}
	 	{-r_s(x),} & {x \in\left[\frac{2 k_s-1}{2^{m}}, \frac{2 k_s+1}{2^m}\right]}, \\\\
	 	{0,} & {\text{otherwise}}.
	 	\end{array}\right.
	 	\]
	 At last we let $\^A_3(\cdot) = -1$ and 
	 \[
	 \^W_3(x) = -
	 \begin{bmatrix}
	 \sum_{i=1}^2 x_i\\\vdots\\\sum_{i=2^{m-1}}^{2^m}x_i
	 \end{bmatrix},
	 \quad \text{then }
	 \^M_3(x) = \max\{-1, \^M_2(x)\} = 
	 \begin{bmatrix}
	 r_1\\ r_{2}\\\vdots\\r_{m}
	 \end{bmatrix}.
	 \]
	 We define the above three layer MCN as $M_0(x) \coloneqq \^M_3(x)$.
	 Then we have a $(l+3)$-layer MCN $M(x)$ such that
	 \[
	 M(x) \coloneqq M_l \circ M_{l-1} \circ \cdots \circ M_0(x), \quad \text{with } \^{\tilde{A}}_{l}(x) = x \text{ and } \^{\tilde{A}}_{k}(x) = 0, \forall k\leq l.
	 \]
	 It is obvious that the first entry of $M(x) = x+\sum_{i=1}^{ml} r_i(x)$. Form the previous results, e.g., Proposition 2 in~\cite{yarotsky2017error} and Lemma A.1. in~\cite{schmidt2017nonparametric}, we already have
	 \[
	 \left|x+\sum_{i=1}^{ml} r_i(x) - x^2\right| \leq 2^{-ml} \leq 2^{-wl}
	 \]
	 We can easily find that the number of the non-zero parameters for $M(x)$ is in the order of $O(w^2l + 2^m)$. However, since MCN has the concatenation operator as in the Eq. (\ref{Eq:LM-CVNN}), we can expand the width of $M_i$ so that the  $w = \operatorname{dim}(M_i) >> m$, when $l$ is large, we can have $w^2l\geq 2^m$; and finish the proof.
	 \end{proof}
Note that 
\[
x y=\frac{1}{2}\left((x+y)^{2}-x^{2}-y^{2}\right),
\]
we can use Claim \ref{clm:apprx_x2} to efficiently approximate polynomial by MCN. 
\begin{claim}\label{clm:appro_poly}
	The function $f(\*x) = \prod_{i=1}^p x_i$ on $[-1,1]^p$ can be approximated by an MCN $\tilde{M}_p(\*x)$ of width $\^O(wp)$ and depth $\^O(l\ln p)$, with the error bound as: 
	\[
	\left|\tilde{M}_p(\*x) - \prod_{i=1}^p x_i\right| \leq \^O(p2^{-wl}).
	\]
	The number of non-zero parameters for this MCN is in the order of $\^O(pw^2l)$.
\end{claim}
\begin{proof}
	We already have a $(l+3)$-layer MCN $M(x)$ such that can approximate $x^2$ accurately. We can easily get a $(l+5)$-layer modified MCN $\tilde{M}(x,y)$ such that  $\tilde{M}(x,y) \approx xy$. $\tilde{M}$ can be obtained by
	\[
	\tilde{M}(x,y) \coloneqq
	\begin{bmatrix}
	x\\y
	\end{bmatrix}
	\to
	\begin{bmatrix}
	x+y\\x\\y
	\end{bmatrix}
	\to
	\begin{bmatrix}
	M(x+y)\\M(x)\\M(y)
	\end{bmatrix}
	\to
	\frac{1}{2}\left(M(x+y)-M(x)-M(y)\right),	
	\]
	It is obvious that
	\[
	|\tilde{M}(x,y)-xy| \leq \frac{3}{2} \cdot 2^{-wl} = \^O(2^{-wl})\coloneqq \epsilon,
	\]
	and the number of non-zero parameters for $\tilde{M}(x,y)$ is also in the order of $O(w^2l)$.
	Based on the above observation, we can construct an MCN such that
	approximate $\prod_{i=1}^p x_i$. Denote $i\coloneqq \left\lceil\log _{2}(p)\right\rceil$. In the first layer, we computer
	\[
	\*x \to \left[x_1,\cdots,x_p,\underbrace{1,\cdots,1}_{2^i - q}\right]^\top\coloneqq \*y,
	\]
	then we define the multivariate version $\tilde{M}(\*y)$ for  $\*y \in \#R^{2^j}$, whre $j \in \#N_+$,
	\[
	\tilde{M}(\*y) \coloneqq 
	 \*y \to
	\begin{bmatrix}
	\tilde{M}(y_1,y_2)\\\vdots\\\tilde{M}(y_{2^j},y_{2^j-1})	\end{bmatrix}.
	\]
	Then we can have a $(li+5i+1)$-layer MCN $\tilde{M}_p(\*x)$, with the width be $wp$, such that,
	\[
	\tilde{M}_p(\*x)\coloneqq 
	\underbrace{\tilde{M}\circ  \cdots \circ \tilde{M}(\*y)}_{i}.
	\]
	Note that, for $a,b,c,d \in [-1,1]$, we have
	\[
	\tilde{M}(a,b) - cd \leq \epsilon + |a-c| + |b-d|.
	\]
	Recall that $i\coloneqq \left\lceil\log _{2}(p)\right\rceil$ and omit the high order terms of $\epsilon$, we get
	\[
	\left|\tilde{M}_p(\*x) - \prod_{i=1}^p x_i\right|\leq \sum_{k=0}^{i-1} 2^{k} \epsilon \leq 2^i \epsilon = \^O(p2^{-wl}).
	\]
	It is easy to verify that the number of non-zero parameters is in the order of $\^O(2^i + p\cdot w^2l) = \^O(pw^2l)$.	
\end{proof}

\begin{claim}\label{clm:appro_poly2}
	The function $f(x) = \sum_{j=1}^p a_j x^j$, where and $  x \in[-1,1]$, can be approximated by MCN $M_{\operatorname{poly}}$ of width $\^O(wp\ln p)$ and depth $\^O(l\ln p)$, with the error bound as: 
	\[
	\left|\tilde{M}_{\operatorname{poly}}(x) - \sum_{j=1}^p a_j x^j\right| \leq \^O(\|\*a\|_1p^22^{-wl}).
	\]
	The number of non-zero parameters for this MCN is in the order of $\^O(w^2lp\ln p )$.
\end{claim}
\begin{proof}
	We first the copy $x$ $p$-times
	\[
	\*x_p \coloneqq [\underbrace{x,\cdots,x}_{p}]^\top.
	\] 
	We then apply the MCN $\tilde{M}_p(\*x_p)$ in Claim \ref{clm:appro_poly} to it to approximate $x^p$. 
	Interestingly, since MCN has the skip-connection with any previous layer by the operator $\^A_k()$, hence from the MCN $\tilde{M}_p(\*x_p)$ in Claim \ref{clm:appro_poly} we can extract
	\[
	\*y \coloneqq
	\left[\tilde{M}_1(\*x_p),\tilde{M}_2(\*x_p),\tilde{M}_4(\*x_p),\cdots,\tilde{M}_{2^i}(\*x_p)\right] \approx \left[x,x^2,x^4,\cdots,x^{2^i}\right],
	\]
	where $i\coloneqq \left\lceil\log _{2}(p)\right\rceil$.
	We now append $p$ sub-MCNs on $\*y$ to approximate $x^j$ for $j=1,\cdots,p$. Each sub-MCN first need to choose components from $\*y$, then use the MCN $\tilde{M}_p(\cdot)$ in Claim \ref{clm:appro_poly}  to ``multiply'' the components, e.g., 
	\[
	x^7 = x\cdot x^2\cdot x^4 \approx \tilde{M}_3\left(
	\begin{bmatrix}
	\tilde{M}_1(\*x_p)\\\tilde{M}_2(\*x_p)\\\tilde{M}_4(\*x_p)	\end{bmatrix}
	\right).
	\]
	By the property of telescoping sum and the results in the previous Claim, the approximation error for $x^7$ is in the order $\^O(3\cdot2^{-wl} + 3 \cdot2^{-wl})$. Actually, finding such a sub-MCN for $x^j$ is equivalent to expressing $j$ in binary. Hence, the approximation error for each sub-MCN $\tilde{M}^{\operatorname{sub}}_j$ which aims at $x^j$ is
	\[
	\left|\tilde{M}^{\operatorname{sub}}_j(\*y) -x^j\right| \leq \^O\left(\ln p \cdot 2^{-wl} + (\sum_{k=0}^i 2^k) \cdot 2^{-wl}\right) = \^O(p 2^{-wl}).	
	\]
	Therefore, Let
	\[
	\tilde{M}_{\operatorname{poly}}(\*x) \coloneqq \sum_{j=1}^p a_j \tilde{M}^{\operatorname{sub}}_j(\*y),
	\]
	then
	\[
	\left|\tilde{M}_{\operatorname{poly}}(\*x) - \sum_{j=1}^p a_j x^j\right| \leq \^O(\|\*a\|_1p^22^{-wl}).
	\]
	The total number of non-zero parameters for $\tilde{M}^{\operatorname{sub}}_j(\*y)$ is in the order of
	\[
	\^O\left(\sum_{k=1}^i \binom{i}{k} k w^2l\right) = \^O(p\ln p w^2l).
	\]
	Hence, by adding the parameters in MCN $\tilde{M}_p(\*x_p)$ which maps $\*x_p$ to $\*y$, the non-zero parameters of $\tilde{M}_{\operatorname{poly}}(\*x)$ is in the order
	\[
	\^O(p\ln p w^2l + 2p+ \ln p + p w^2l) = \^O( w^2lp\ln p).
	\]
\end{proof}

\begin{claim}\label{clm:appro_sin}
	The function $f(\*x) = \cos(n\pi x)$ or $f(\*x) = \sin\left((n-\frac{1}{2})\pi x\right)$, where $n \in \mathbb{N}_{+}=\mathbb{N} \backslash \{0\}$ and $  x \in[-1,1]$, can be approximated by MCNs $M_{\cos}$ and $M_{\sin}$ of width $\^O(wp\ln p)$ and depth $\^O(l\ln p + n^2)$, with the proper activation function and the error bound is: 
	\[
	\epsilon = \^O(p^{-p}\exp(p) + p^2 2^{-wl}).
	\]
	The number of non-zero parameters for this MCN is in the order of $\^O\left(w^2lp\ln p+n^2 \right)$.
\end{claim}
	\begin{proof}
		We first consider the case $n=1$ for $\cos(n\pi x)$. Let $y \coloneqq \pi x$, then $y\in [-\pi,\pi]$. We now need to construct an MCN to approximate $\cos(y)$ on the interval $[-\pi,\pi]$. First, we can divide the interval $[-\pi,\pi]$ into several sub-intervals and each sub-interval has the length smaller than $1$, e.g., $[0,\pi/4]$ and $[\pi/4,\pi/2]$. Then we perform the Taylor expansion on each sub-interval, say $[0,\pi/4]$  for example. Since the derivative of $\cos(y)$ up to any order is bounded, the proof for other sub-interval share a similar roadmap. Note that
		\[
		\cos(y) = \sum_{n=0}^{\infty}(-1)^{n} \frac{y^{2 n}}{(2 n) !},
		\]
		Hence, when the even number $p$ is large, we have
		\[
		\left|\cos(y) - \sum_{n=0}^{p}(-1)^{n} \frac{y^{2 n}}{(2 n) !}\right| \leq \^O(\frac{|y|^p}{p!}) \leq \^O(\frac{1}{p!}) = \^O\left(p^{-p-\frac{1}{2}} \exp(p)\right),
		\]
		where the last equality comes from the Stirling's formula. Based on the results in Claim \ref{clm:appro_poly2}, there exists an MCN $M^{n=1}_{\cos}$ such that
		\[
		\tilde{M}^{n=1}_{\cos} \approx \sum_{n=0}^{p}(-1)^{n} \frac{y^{2 n}}{(2 n) !},\quad \forall x\in [0,\pi],
		\]
		with the approximation error in the order $\^O(p^22^{-wl}\exp(1))$,
		hence we parallelize all the MCNs $\tilde{M}^{n=1}_{\cos}$ on each sub-interval and obtain a final MCN $M^{n=1}_{\cos}$ of width $O(wp\ln p)$ and depth $O(l \ln p)$ such that
		\[
		\left|M^{n=1}_{\cos}(x) -\cos(\pi x)\right| \leq \^O(p^{-p}\exp(p) + p^2 2^{-wl})\coloneqq \epsilon.
		\]
		By the periodicity of $\cos(x)$, we have
		\[
		\cos (n \pi x) = \cos(n\pi x - \lfloor \frac{xn^2}{2} \rfloor \frac{2\pi}{n}),
		\]
		where $\lfloor \cdot \rfloor$ is the floor operator. We now need to construct an MCN which can exact perform the floor operator. Actually this can be easily implemented by choosing proper activation. Let the activation be the binary step function:
		\[
		\sigma(x)=\left\{\begin{array}{ll}
		{0}, & {\text { for } x<0}, \\
		{1}, & {\text { for } x \geq 0}.
		\end{array}\right.
		\]
		Then we can obtain the floor operator on the interval $[0,n^2/2]$ by an MCN $M_f$ of width $O(1)$ and depth $O(n^2)$
		\[
		M_f(x) = \sum_{j=1}^{\lfloor\frac{n^2}{2}\rfloor} \sigma(x-j),
		\]
		By the oddness of the floor operator, we can obtain $\lfloor y \rfloor$ for $y \in [-n^2/2,0]$ without adding the depth.
		Hence, we can have an MCN $M^{n}_{\cos}$ of width $O(wp\ln p)$ and depth $O(l \ln p + n^2)$ such that
		\[
		\left|M^{n}_{\cos}(x) - \cos(n\pi x)\right| \leq  \epsilon ,\quad \forall x \in [1,1].
		\]
		It is obvious that the number of the non-zero parameters of $M^{n}_{\cos}(x)$ is in the order
		\[
		\^O\left(w^2lp\ln p+n^2 \right).
		\]
		Note that we can get the approximation of $\cos(k\pi x)$ for all $k = 1,\cdots,n$ from the intermediate layers of $M^{n}_{\cos}(x)$ without recalculation. 
		Recall the definition of the Dirichlet kernel, we have
		\[
		1+2 \cos x+2 \cos 2 x+2 \cos 3 x+\cdots+2 \cos (n x)=\frac{\sin \left[\left(n+\frac{1}{2}\right) x\right]}{\sin \frac{x}{2}},
		\]
		Hence, we can easily obtain the approximation of $\sin\left((n-\frac{1}{2})\pi x\right)$ based on the intermediate layers of MCN $M^{n}_{\cos}(x)$ without add the size of network.
	\end{proof}
	Now let 
	\[
	\phi_{0}^{[0]}(x)=\frac{1}{\sqrt{2}},\quad \phi_{n}^{[0]}(x)=\cos (n \pi x), \quad \phi_{n}^{[1]}(x)=\sin \left(\left(n-\frac{1}{2}\right) \pi x \right),
	\]
	where
	\[
	n \in \mathbb{N}_{+},\quad x \in[-1,1].
	\]
	Given multi-indices $\*n = (n_1,\cdots,n_d) \in \#N^d$ and $\*i = (i_1,\cdots,i_d) \in \{0, 1\}^d$, we define a d-variate functions
	\[
	\phi_{\*n}^{[\*i]}(\*x)=\prod_{j=1}^{d} \phi_{n_{j}}^{\left[i_{j}\right]}\left(x_{j}\right), \quad \*x=\left(x_{1}, \ldots, x_{d}\right) \in[-1,1]^{d}.
	\]
	From a standard result of spectral theory, the set $\left\{\phi_{\*n}^{[\*i]}:\* n \in \mathbb{N}^{d}, \*i \in\{0,1\}^{d}\right\}$ is an orthonormal basis of $\mathrm{L}^{2}(-1,1)^{d}$. We can also construct MCNs which approximate $\phi_{\*n}^{[\*i]}(\*x)$ well.
	\begin{claim}\label{clm:appro_phi}
		The function $\phi_{\*n}^{[\*i]}(\*x)$ can be approximated by MCNs $M_{\phi}$ of width $\^O(dwp\ln p)$ and depth $\^O(l\ln p + \|\*n\|^2_\infty )$, with the error bound as: 
		\[
		\left|M_{\phi}(\*x) - \phi_{\*n}^{[\*i]}(\*x)\right| = \^O\left(d\left(p^{-p}\exp(p) + p^2 2^{-wl}\right)\right).
		\]
		The number of non-zero parameters for this MCN is in the order of $\^O\left(dw^2p\ln p+ \|\*n\|^2_2\right)$.
	\end{claim}
\begin{proof}
	For each entry of the vector $\*x$, we append the MCNs $M^{(n=n_j)}_{\cos}$ or $M^{(n=n_j)}_{\sin}$ form the Claim \ref{clm:appro_sin} to approximate the function $\phi_{n_{j}}^{\left[i_{j}\right]}\left(x_{j}\right)$. Then, we ``multiply'' the functions $\phi_{n_{j}}^{\left[i_{j}\right]}\left(x_{j}\right)$ at the last layer by the MCN in Claim \ref{clm:appro_poly}, hence the  approximation error  is
	\[
	\^O\left(d\left(p^{-p}\exp(p) + p^2 2^{-wl}\right)+p2^{-wl}\right) = 	\^O\left(d\left(p^{-p}\exp(p) + p^2 2^{-wl}\right)\right),
	\]
	while $M_{\phi}(\*x)$ is in the width $\^O(dwp\ln p)$ and depth $\^O(l\ln p + \|\*n\|^2_\infty )$. We sum all the parameters in the $M^{(n=n_j)}_{\cos}$ or $M^{(n=n_j)}_{\sin}$, the non-zero parameters for $M_{\phi}(\*x)$ is in the order of
	\[
	\^O\left(dw^2p\ln p+ \|\*n\|^2_2\right).
	\]
\end{proof}
Claim \ref{clm:appro_phi} shows that there exists MCNs $M_{\phi}$ such can approximate the orthonormal basis of $\mathrm{L}^{2}(-1,1)^{d}$ well.
\par
For a function $\*f \in \mathrm{L}^{2}(-1,1)^{d}$, a truncation parameter $N \in \#N$ and finite index set $I_N \in \#N^d$, we can get the truncated Fourier series of $\*f$
\[
\mathcal{F}_{N}[\*f](\*x)=\sum_{\*i \in[0,1]^{d},~ \*n \in I_{N}} \hat{\*f}_{\*n}^{[\*i]} \phi_{\*n}^{[\*i]}(\*x), \quad \text { where } \quad \hat{\*f}_{\*n}^{[\*i]}=\int_{(-1,1)^{d}} \*f(\*x) \phi_{\*n}^{[\*i]}(\*x) \mathrm{d} \*x.
\]
Before preceding, we provide a previous result to bound the Fourier coefficients.
\begin{lemma}\label{lem:Fourier_bound}
	Suppose that $\*f$ satisfy the Condition $\ref{cond:r2f}$. Then
	\[
	\left|\hat{\*f}_{\*n}^{[\*i]}\right|\leq C(\chi(n),d,k)
	\left(\bar{n}_{1} \cdots \bar{n}_{d}\right)^{-2(s+1)}\|f\|_{2 s+2, \^H}, \quad \*n \in \#{N}^{d}
	\]
	where $\bar{m}=\max \{m, 1\}$ for $ m \in \mathbb{N}$, $C(\chi(n),d,k)$ is a constant only depends on the $\chi(n)$ (the number of non-zero entries in $\*n$), dimension $d$ and the smoothness of $\*f$; and
	\[
	\|\*f\|_{s, \^H}^{2}=\sum_{\|\bm{\alpha}\|_{\infty} \leq s}\left\|\mathrm{D}^{\bm{\alpha}} \*f\right\|^{2},
	\]
\end{lemma}
\begin{proof}
	The proof can be found in~\cite{olver2009convergence} and Theorem 2.14 in~\cite{adcock2010multivariate}.
\end{proof}
We now suppose that $N = 2^r$ and let
\[
I_N = \bigcup_{\|\bm{\alpha}\|_1 \leq r} \rho(\bm{\alpha}),
\]
where
\[
\rho(\bm{\alpha})=\left\{\*n \in \#N^d:\left\lfloor 2^{\alpha_{j}-1}\right\rfloor \leq n_{j}<2^{\alpha_{j}}, ~~ j=1, \cdots, d\right\}, \quad \bm{\alpha} \in \#N^d.
\]
We consider the size of $I_N$ in the following lemma.
\begin{lemma}\label{lem:IN_bound}
	The number of terms in the set $I_N$ is
	\[
	\frac{N(\ln N)^{d-1}}{(d-1) !} +\^O\left(N(\ln N)^{d-2}\right).
	\]
\end{lemma}
\begin{proof}
	The proof for the size of $I_N$ can be found in \cite{huybrechs2011high}.
\end{proof}
We now provide the asymptotic order of $\mathcal{F}_{\bm{\alpha}}[\*f](\*x)$.
\begin{lemma}\label{lem:f_alpha_bound}
	Suppose that $\*f$ satisfy the Condition $\ref{cond:r2f}$.
	Let
	\[
	\mathcal{F}_{\bm{\alpha}}[\*f](\*x)=\sum_{\*i \in\{0,1\}^{d}} \sum_{\*n \in \rho(\bm{\alpha})} \hat{\*f}_{\*n}^{[\*i]} \phi_{\*n}^{[\*i]}(\*x),\quad \bm{\alpha} \in \#N^d.
	\]
	Then we have
	\[
	\mathcal{F}_{\bm{\alpha}}[\*f](\*x)=\mathcal{O}\left(2^{-2(s+1)\|\bm{\alpha}\|_1}\right), \quad\|\bm{\alpha}\|_1 \rightarrow \infty.
	\]
\end{lemma}
\begin{proof}
	The proof for the asymptotic order of $\mathcal{F}_{\bm{\alpha}}[\*f](\*x)$ refers to the Eq. (4.8) in~\cite{adcock2010multivariate}.
\end{proof}
Now all the things are ready, we first consider the reminder of $\mathcal{F}_{N}[\*f](\*x)$
\[
\begin{split}
\left|\*f -\mathcal{F}_{N}[\*f](\*x)\right|
& = \sum_{\|\bm{\alpha}\|_1>r} \mathcal{F}_{\bm{\alpha}}[\*f](\*x)
=\^O\left( \sum_{\|\bm{\alpha}\|_1>r} \left(2^{-2(s+1)\|\bm{\alpha}\|_1}\right)\right)\\
&= \^O\left(\int_{\|\bm{\alpha}\|_1>r} 2^{-2(s+1)\|\bm{\alpha}\|_1}\right).
\end{split}
\]
Let $t_1^2 = |\alpha_1|, t_2^2 = |\alpha_2|,\cdots,t_d^2 = |\alpha_d|$, then we have
\[
\begin{split}
&\int_{\|\*t\|^2_2>r} 2^{-2(s+1)\|\*t\|^2_2} \\
=&\int_{\varphi_{d-1}=0}^{2 \pi} \int_{\varphi_{d-2}=0}^{\pi} \cdots \int_{\varphi_{1}=0}^{\pi} \int_{\tilde{r}=\sqrt{r}}^{\infty} 
2^d \prod_{i=1}^d t_i\cdot
2^{-2(s+1)\tilde{r}^2}
\tilde{r}^{d-1} \sin ^{d-2}\left(\varphi_{1}\right) \sin ^{d-3}\left(\varphi_{2}\right) \cdots \sin \left(\varphi_{d-2}\right) d \tilde{r} d \varphi_{1}  \cdots d \varphi_{d-1}\\
=&
2^d\int_{0}^{2 \pi} \cdots \int_{0}^{\pi} \int_{\sqrt{r}}^{\infty} 
2^{-2(s+1)\tilde{r}^2}\tilde{r}^{2d-1}
\prod_{i=1}^d \cos \left(\varphi_{i}\right)
\sin^{2d-3}\left(\varphi_{1}\right) \sin^{2d-5}\left(\varphi_{2}\right)\cdots 
\sin^3\left(\varphi_{d-2}\right)
\sin\left(\varphi_{d-1}\right)
d \tilde{r} d \varphi_{1}  \cdots d \varphi_{d-1}\\
=&\^O\left( \int_{\sqrt{r}}^{\infty} 2^{-2(s+1)\tilde{r}^2}\tilde{r}^{2d-1} d \tilde{r} \right)
=\^O\left( \int_{r}^{\infty} 2^{-2(s+1)u}u^{d-1} du \right)
=\^O\left(2^{-2(s+1)r}r^{d-1}\right).
\end{split}
\]
Hence, we can get
\[
\left|\*f -\mathcal{F}_{N}[\*f](\*x)\right| = \^O\left(2^{-2(s+1)r}r^{d-1}\right) = \^O\left(N^{-2s-2}\left(\ln N \right)^{d-1}\right).
\]
We then consider the approximation error for $\mathcal{F}_{N}[\*f](\*x)$ by MCN.
By Lemma \ref{lem:Fourier_bound}, we know that 
\[
\left|\hat{\*f}_{\*n}^{[\*i]}\right| = \^O\left(\bar{n}_{1} \cdots \bar{n}_{d}\right)^{-2(s+1)}.
\]
Similar to the proof of the reminder term, we use the power of $2$ to represent $\bar{n}_{i}=2^{\alpha_i}$ for $i = 1,\cdots,d$, then
\[
\sum_{\*i \in[0,1]^{d},~ \*n \in I_{N}} \hat{\*f}_{\*n}^{[\*i]} \leq 
\sum_{\*i \in[0,1]^{d},~ \*n \in I_{N}} |\hat{\*f}_{\*n}^{[\*i]}| \leq 
2^d\sum_{\|\bm{\alpha}\|_1\leq r} 2^{-2(s+1)\|\bm{\alpha}\|_1} \leq
\^O\left(2^d\int_{\|\bm{\alpha}\|_1\leq r} 2^{-2(s+1)\|\bm{\alpha}\|_1}\right).
\]
By a similar calculation above, we can get
\[
\^O\left(2^d\int_{\|\bm{\alpha}\|_1\leq r} 2^{-2(s+1)\|\bm{\alpha}\|_1}\right) = 
\^O\left(2^d \int_{0}^{r} 2^{-2(s+1)u}u^{d-1} du \right)
=\^O\left(2^d\right).
\]
Note that, when $p$ is large, there exists MCNs $M_{\phi}(\*x)$ such that
\[
\left|M_{\phi}(\*x) - \phi_{\*n}^{[\*i]}(\*x)\right| = \^O\left(d\left(p^{-p}\exp(p) + p^2 2^{-wl}\right)\right) =
\^O\left(dp^2 2^{-wl}\right)
\coloneqq \epsilon.
\]
Hence we combine all the MCNs $M_{\phi}(\*x)$ together to get  an MCN $M_{\mathcal{F}_{N}}$ such that
\[
\left|M_{\mathcal{F}_{N}} - \mathcal{F}_{N}[\*f](\*x)\right|\leq
\sum_{\*i \in[0,1]^{d},~ \*n \in I_{N}} |\hat{\*f}_{\*n}^{[\*i]}| \epsilon 
= \^O\left(d2^dp^2 2^{-wl}\right).
\]
For any $\*n$ with strictly positive entries there are $2^d$ choices of $\*i \in \{0,1\}^d$.
The total number of coefficients $\hat{\*f}_{\*n}^{[\*i]}$ where at least one entry of $n$ is zero is $\^O(N(\ln N)^{d-1}$ by Lemma \ref{lem:IN_bound}. Hence the total number of the coefficient in $\mathcal{F}_{N}[\*f](\*x)$ is in the order of
\[
\frac{2^{d}}{(d-1) !} N(\ln N)^{d-1}+\mathcal{O}\left(N(\ln N)^{d-2}\right).
\]
When $d$ is large, by the Stirling's formula, we have 
\[
\frac{2^{d}}{(d-1) !} \to 0,\quad \text{as } d \to \infty.
\]
Hence, we have $\^O\left(N(\ln N)^{d-2}\right)$ MCNs $M_{\phi}(\*x)$ to combine. The MCN $M_{\mathcal{F}_{N}}$ is in the width of $\^O\left(N(\ln N)^{d-2} dwp\ln p\right)$ and depth of $\^O\left(l\ln p + N^2\right)$, or have $\^O\left( dwp\ln p\right)$ width and  $\^O\left(N(\ln N)^{d-2}l\ln p + N^3(\ln N)^{d-2}\right)$ depth. It is obvious that the non-zero parameters for $M_{\mathcal{F}_{N}}$ is in the order of 
\[
\^O\left(N\left(\ln N\right)^{d-2}\left(dw^2p\ln p + \|\*n\|_2^2\right)\right) \leq
 \^O\left(N\left(\ln N\right)^{d-2}\left(dw^2p\ln p + N^2\right)\right).
\]
We now finish the proof.
\end{proof}

\subsection{Proof of Theorem \ref{thm:covering}}
\begin{proof}
As shown in Eq.~(\ref{Eq:LM-CVNN}) that $\*x_{k+1} = \left[	\^L_{k+1}(\*x_k);~\^{\tilde{A}}_{k+1}(\*x_0) + \max\left\{\^W_{k+1}(\*x_k),  \sigma_{k+1}\left(\^A_{k+1}\left(\*x_{\hat{k}}\right)\right)\right\}\right]$, by introducing an auxiliary variable $
\*y_{k}$, MCN can be reformulated as a nested function as follows:
\begin{equation}\label{nested_func}
	\*y_{k+1} = \^G_{k}(\*y_{k}) = \operatorname{concate}\left(\*y_{k}, \sigma_{\operatorname{MCN}}\left( \^T_{k+1}(\*y_{k}) \right)\right)  = \left[ \*y_{k}; \sigma_{\operatorname{MCN}}\left( \^T_{k+1}(\*y_{k}) \right)\right], \qquad \*y_{0} = \*x_{0},
\end{equation}
where $\^T_{k+1}(\cdot)$ is a Block Sparse Operator Matrix and $\*y_{k}$ is a column vector consist of all entries from $\*x_{0}$ to $\*x_{k}$, which are defined as follows:
\[
\^T_{k+1}(\cdot)=\begin{bmatrix}
\^{\tilde{A}}_{k+1} & \ldots & \^O & \ldots& \^O\\
\^O & \ldots & \^A_{k+1} & \ldots &\^O\\
\^O & \ldots & \^O  & \ldots&\^W_{k+1}\\
\^O & \ldots & \^O & \ldots &\^L_{k+1}\\
\end{bmatrix},\qquad
\*y_{k} = \begin{bmatrix}
\*x_{0}\\ \vdots \\ \*x_{k}
\end{bmatrix}.
\]
It should be mentioned that, each row of $\^T_{k+1}(\cdot)$ only has one non-zero block at $\hat{k}$-th column, the index of which is determined by the structure of each MCN block. And we use a concatenate vector $\*y_{k}$ to integrate different subscripts $\hat{k}$.

Moreover, $\sigma_{MCN}$ in Eq.~(\ref{nested_func}) is a special activation function corresponding to Eq.~(\ref{Eq:LM-CVNN}):
\[
\sigma_{\operatorname{MCN}}\left( \begin{bmatrix}
a\\b\\c\\d
\end{bmatrix} \right) = \left[ d; a + \max\left\{ \sigma_{k+1}(b),c \right\} \right].
\]
\begin{claim}\label{clm:lipMCN}
The operators $\^G_{k}(\cdot)$, $\sigma_{MCN}$, and $\^T_{k+1}$ in Eq.~(\ref{nested_func}) are Lipschitz continuous w.r.t. $\ell_1$ norm.
Moreover, the Lipschitz constant for the operator $\^G_{k}(\cdot)$ is 
\[
\kappa_k \coloneqq \left( 1  +  \max\{\rho_{k+1},2 \} \max\{ \|\^{\tilde{A}}_{k+1}\|_{1}, \|\^A_{k+1}\|_{1}, \|\^W_{k+1}+\^L_{k+1}\|_{1} \}\right),
\]
where $\|\cdot\|_{1}$ is the operator $\ell_1$ norms induced by vector $\ell_1$ norms
\[
 \|\^A\|_{1} = \max_{x \ne 0} \frac{\|\^A(\*x)\|_{1}}{\|\*x\|_1}. 
\]
\end{claim}
\begin{proof}
	Let $\*y_{k} = \begin{bmatrix}
	\*x_{0}\\ \vdots \\ \*x_{k}
	\end{bmatrix}$ and $\*y'_{k} = \begin{bmatrix}
	\*x'_{0}\\ \vdots \\ \*x'_{k}
	\end{bmatrix}$, then we have $\^T_{k+1}(\*y_{k})=\begin{bmatrix}
	\^{\tilde{A}}_{k+1}(\*x_{0})\\
	\^A_{k+1}(\*x_{\hat{k}})\\
	\^W_{k+1}(\*x_{k})\\
	\^L_{k+1}(\*x_{k})
	\end{bmatrix}$ and  $\^T_{k+1}(\*y_{k})=\begin{bmatrix}
	\^{\tilde{A}}_{k+1}(\*x'_{0})\\
	\^A_{k+1}(\*x'_{\hat{k}})\\
	\^W_{k+1}(\*x'_{k})\\
	\^L_{k+1}(\*x'_{k})
	\end{bmatrix}$.
	
	It can be seen that $\^T_{k+1}(\cdot)$ is a Lipschitz continuous function w.r.t. $\ell_1$ norm. By the definition of Lipschitz continuity and induction norm, it is easy to check that 
	$L_{\^T_{k+1}}=\max\{ \|\^{\tilde{A}}_{k+1}\|_{1}, \|\^A_{k+1}\|_{1}, \|\^W_{k+1}+\^L_{k+1}\|_{1} \}$.
	
	For convenience, we use $\*p,\*q,\*r,\*s$ and $\*p',\*q',\*r',\*s'$ to denote each entries of $\^T_{k+1}(\*y_{k})$ and $\^T_{k+1}(\*y'_{k})$, which means $\^T_{k+1}(\*y_{k}) = \begin{bmatrix}
	\*p\\\*q\\\*r\\\*s
	\end{bmatrix}$ and  $\^T_{k+1}(\*y'_{k}) = \begin{bmatrix}
	\*p'\\\*q'\\\*r'\\\*s'
	\end{bmatrix}$. Then by using the definition of Lipschitz continuous, we have 
	
	\begin{equation*}
	\begin{aligned}
	\^G_{k}(\*y_{k}) - \^G_{k}(\*y'_{k}) &= \left[ \*y_{k}; \sigma_{\operatorname{MCN}}\left( \^T_{k+1}(\*y_{k}) \right)\right] - \left[ \*y'_{k}; \sigma_{\operatorname{MCN}}\left( \^T_{k+1}(\*y'_{k}) \right)\right]\\
	&= \left[ \*y_{k} - \*y'_{k}; \sigma_{\operatorname{MCN}}\left( \^T_{k+1}(\*y_{k}) \right) - \sigma_{\operatorname{MCN}}\left( \^T_{k+1}(\*y'_{k}) \right)\right]\\
	&= \left[ \*y_{k} - \*y'_{k}; \*s-\*s'; (\*p-\*p') + \left(\max\left\{ \sigma_{k+1}(\*q),\*r\right\} - \max\left\{ \sigma_{k+1}(\*q'),\*r'\right\}\right)\right]\\
	&= \left[ \*y_{k} - \*y'_{k}; \*s-\*s'; (\*p-\*p') + \left(\max\left\{ \sigma_{k+1}(\*q),\*r\right\} - \max\left\{ \sigma_{k+1}(\*q'),\*r'\right\}\right)\right]\\
	&= \left[ \*y_{k} - \*y'_{k}; \*s-\*s'; (\*p-\*p') + \left(\operatorname{ReLU}\left\{ \sigma_{k+1}(\*q)-\*r\right\} + \*r - \operatorname{ReLU}\left\{ \sigma_{k+1}(\*q')-\*r'\right\} - \*r'\right)\right]\\
	&= \left[ \*y_{k} - \*y'_{k}; \*s-\*s'; (\*p-\*p') + (\*r-\*r') + \left(\operatorname{ReLU}\left\{ \sigma_{k+1}(\*q)-\*r\right\} - \operatorname{ReLU}\left\{ \sigma_{k+1}(\*q')-\*r'\right\}\right)\right]\\
	\end{aligned}
	\end{equation*}
	Then
	\begin{equation*}
		\begin{aligned}
		&\left\|\^G_{k}(\*y_{k}) - \^G_{k}(\*y'_{k})\right\|_1\\
		=&  \left|\left[ \*y_{k} - \*y'_{k}; \*s-\*s'; (\*p-\*p') + (\*r-\*r') + \left(\operatorname{ReLU}\left\{ \sigma_{k+1}(\*q)-\*r\right\} - \operatorname{ReLU}\left\{ \sigma_{k+1}(\*q')-\*r'\right\}\right)\right]\right|_1\\
		=& \left\| \*y_{k} - \*y'_{k} \right\|_1 + \left\| \*s-\*s' \right|_1 + \left| (\*p-\*p') + (\*r-\*r') +  \left(\operatorname{ReLU}\left\{ \sigma_{k+1}(\*q)-\*r\right\} - \operatorname{ReLU}\left\{ \sigma_{k+1}(\*q')-\*r'\right\}\right) \right\|_1\\
		\leq& \left\| \*y_{k} - \*y'_{k} \right\|_1 + \left\| \*s-\*s' \right\|_1 + \left\| \*p-\*p' \right\|_1 + \left\| \*r-\*r'\right\|_1 +  \left\|\operatorname{ReLU}\left\{ \sigma_{k+1}(\*q)-\*r\right\} - \operatorname{ReLU}\left\{ \sigma_{k+1}(\*q')-\*r'\right\}\right\|_1 \\
		\leq& \left\| \*y_{k} - \*y'_{k} \right\|_1 + \left\| \*s-\*s' \right\|_1 + \left\| \*p-\*p' \right\|_1 + \left\| \*r-\*r'\right\|_1 +  \left\|\left( \sigma_{k+1}(\*q)-\*r\right) - \left( \sigma_{k+1}(\*q')-\*r'\right)\right\|_1 \\
		\leq& \left\| \*y_{k} - \*y'_{k} \right\|_1 + \left\| \*s-\*s' \right\|_1 + \left\| \*p-\*p' \right\|_1 + 2\left\| \*r-\*r'\right\|_1 +  \left\| \sigma_{k+1}(\*q) -  \sigma_{k+1}(\*q')\right\|_1 \\
		\end{aligned}
	\end{equation*}
	
	Suppose that the activation function $\sigma_{k+1}$ is also Lipschitz continuous with a Lipschitz constant $\rho_{k+1}$, then we have 
	
	\begin{equation}
		\begin{aligned}
		&\left\|\^G_{k}(\*y_{k}) - \^G_{k}(\*y'_{k})\right\|_1\\
		\leq& \left\| \*y_{k} - \*y'_{k} \right\|_1 + \left\| \*s-\*s' \right\|_1 + \left\| \*p-\*p' \right\|_1 + 2\left\| \*r-\*r'\right\|_1 +  \left\| \sigma_{k+1}(\*q) -  \sigma_{k+1}(\*q')\right\|_1 \\
		\leq& \left\| \*y_{k} - \*y'_{k} \right\|_1 + \left\| \*s-\*s' \right\|_1 + \left\| \*p-\*p' \right\|_1 + 2\left\| \*r-\*r'\right\|_1 +  \rho_{k+1}\left\| \*q-\*q' \right\|_1\\
		\leq& \left\| \*y_{k} - \*y'_{k} \right\|_1 + \max\{\rho_{k+1},2 \} \left\| \^T_{k+1}(\*y_{k}) - \^T_{k+1}(\*y'_{k}) \right\|_1\\
		\leq& \left\| \*y_{k} - \*y'_{k} \right\|_1 + \max\{\rho_{k+1},2 \} L_{\^T_{k+1}}\left\| \*y_{k} - \*y'_{k} \right\|_1\\
		=& \left(1 + \max\{\rho_{k+1},2 \} L_{\^T_{k+1}}\right)\left\| \*y_{k} - \*y'_{k} \right\|_1\\
		=& \left( 1  +  \max\{\rho_{k+1},2 \} \max\{ \|\^{\tilde{A}}_{k+1}\|_{1}, \|\^A_{k+1}\|_{1}, \|\^W_{k+1}+\^L_{k+1}\|_{1} \}\right) \left\| \*y_{k} - \*y'_{k} \right\|_1
		\end{aligned}
	\end{equation}	
	From the above, it is easy to get that $\^G_{k}(\cdot)$, $\sigma_{MCN}$, and $\^T_{k+1}$ in Eq.~(\ref{nested_func}) are all Lipschitz functions w.r.t. $\ell_1$ norm.
\end{proof}
Now, given the parameters $\bm{\theta}$ of MCN, we define
\[
\^G_{i\to j}(\bm{\theta}) \coloneqq \^G_{j}\circ\cdots \^G_{i}\circ,\quad 1\leq i \leq j \leq l,
\]
and
\[
\sigma_{\operatorname{C-MCN}}(\*x) = [\*x;\sigma_{\operatorname{MCN}}(\*x)].
\]
For given $\varepsilon > 0$, we consider two MCNs $\*f_{\bm{\theta}_1}$ and $\*f_{\bm{\theta}_2}$ that both are from $\^F(\bm{\theta},S)$ such that $\norm{\bm{\theta}_1 - \bm{\theta}_2}_1 \leq \varepsilon$,
\[
\begin{split}
&\left\|\*f_{\bm{\theta}_1}(\*x) - \*f_{\bm{\theta}_2}(\*x) \right\|_1 
\overset{(a)}{=}\norm{\sum_{k=1}^l \^G_{k+1\to L}(\bm{\theta}_1)\circ\sigma_{\operatorname{C-MCN}}\left(\left(\^T^{(\bm{\theta}_1)}_{k+1}-\^T^{(\bm{\theta}_2)}_{k+1}\right)\left(\^G_{1\to k}(\bm{\theta}_2)\circ \*x\right)\right)}_1\\
\leq &\sum_{k=1}^l \prod_{i=k+1}^l \kappa_i \rho \norm{\left(\^T^{(\bm{\theta}_1)}_{k+1}-\^T^{(\bm{\theta}_2)}_{k+1}\right)\left(\^G_{1\to k}(\bm{\theta}_2)\circ \*x\right)}_1
\leq \rho \sum_{k=1}^l \prod_{i=k+1}^l \kappa_i  \norm{\left(\^G_{1\to k}(\bm{\theta}_2)\circ \*x\right)}_1 \varepsilon\\
\leq & \rho \varepsilon \sum_{k=1}^l \prod_{i=k+1}^l \kappa_i \norm{\left(\^G_{1\to k}(\bm{\theta}_2)\circ \*x\right)}_1 
\leq \rho \varepsilon \sum_{k=1}^l \prod_{i=1}^l \kappa_i  \norm{\*x}_1
\leq  \rho l  \norm{\*x}_1 \prod_{i=1}^l \kappa_i  \varepsilon.
\end{split}
\]
where $(a)$ comes from the Telescoping sum. Thus, for a fixed sparsity pattern $S$ (i.e., the location of nonzero elements in $\bm{\theta}$), the covering number is bounded by 
\[
\left(\frac{\rho l  \norm{\*x}_1 \prod_{i=1}^l l_i }{\delta}\right)^s.\]
Since the number of the sparsity patterns is bounded by
$\binom{w^2l}{s}\leq (w+1)^{ls}$, the log of covering number is bounded above by
\[
\ln \left((w+1)^{ls} \left(\frac{\rho l  \norm{\*x}_1 \prod_{i=1}^l \kappa_i }{\delta}\right)^s \right) \leq \^O\left(ls \ln\left(\frac{\rho  \norm{\*x}_1 \prod_{i=1}^l \kappa_i}{\delta}\right)  \right).
\]
\end{proof}

\subsection{Proof of Theorem \ref{thm:gene}}
\begin{proof}
	Without loss of generality, we assume $d_y = 1$ and let the smoothness parameter $\beta=1$ in this proof, and the proof can be easily extended to high dimensional and general $\beta$ case . We denote the estimator $\*f_{\bm{\theta}}$ as $f_{\bm{\theta}}$ in the following.
	We denote by  $g$ the target function, since it is smooth, we assume that  $g$ has bounded derivative. We also let the compact set $\^C$ be the input domain in this proof.
	\par
	Since the objective $L(\bm{\theta})$ obtains its optimum value on the training set, MCN fits all the training data, i.e., $f_{\bm{\theta}}(\*x_i) = y_i,~\forall i \in [n]$. Hence, $f_{\bm{\theta}}$ is an estimator that interpolates the training data.
	\par
	In the exactly fitting case, we know that $f_{\bm{\theta}}$ partitions the compact set $\^C$ into many nondegenerate subsets. On each subset $\^C_s$, we have
	\[f_{\bm{\theta}}(\*x) = \*w_s^\top \sigma\left(\*A_s\*x\right) + \*b_s^\top \*x,\quad \forall \*x \in \^C_s,\]
	where $\*A_s$ has different shapes for different subsets $\^C_s$, for brevity, we let $\*A_s \in \mathbb{R}^{d_a\times d_x}$.
	Each $\*x \in \^C$ is contained in at least one of these subsets; let $\^V(\*x)$ denote the set of training data points $\{\*x_{(1)},\cdots, \*x_{(|v|)}\}$ that determine the function surface of $f_{\bm{\theta}}(\cdot)$  on this subset $\^C_s$ containing $\*x$, where $d_x \leq |v| \leq (d_x+1)d_a + d_x$.
	\par
	Consider the following linear equation:
	\[
	\left[
	\begin{matrix}
	g(\*x_{(1)})+\varepsilon_1& \cdots & g(\*x_{(|v|)})+ \varepsilon_{|v|}\\
	1 & \cdots & 1 \\
	\end{matrix}
	\right]\*{\overline{w}} =
	\left[
	\begin{matrix}
	f_{\bm{\theta}}(\*x)\\
	1 \\
	\end{matrix}
	\right].
	\]
	where $\varepsilon_i$'s are the noise terms and are i.i.d. Gassuian.
	\begin{claim}
		With high probability, $\*{\overline{w}}$ exists and for some constant $C_w>0$, we have:
		\begin{equation}\label{eq:boundW}
		\|\*{\overline{w}}\|^2 \leq \frac{C_w}{|v|}.
		\end{equation}
	\end{claim}
	\begin{proof}
		Let
		\[
		\left[
		\begin{matrix}
		g(\*x_{(1)})+\varepsilon_1& \cdots & g(\*x_{(|v|)})+ \varepsilon_{|v|}\\
		1 & \cdots & 1 \\
		\end{matrix}
		\right] :=
		\left[
		\begin{matrix}
		\*G\\
		\*1
		\end{matrix}
		\right]	,
		\]
		and denote $\sigma_{\min}(\*G)$ as the minimal singular value of the matrix $\*G$. By Corollary 3.1.3 of~\cite{horn1991topics}, we have:
		\[
		\sigma_{\min}\left(		\left[
		\begin{matrix}
		\*G\\
		\*1
		\end{matrix}
		\right]	\right) \geq \sigma_{\min}\left(\*G\right).
		\]
		Note that the column of  matrix $\*G$ is bounded by $C_G$ (i.e., the $\ell_2$-norm of each column is upper bounded), without loss of generality, we assume that $C_G$ is  small, otherwise, we can divide all the function values $g(\*x)$ by a large constant.
		Moreover, w.o.l.g. we let $\mathbb{E}[\*g(\cdot)] = \*I_{d_y}$. i.e., each dimension of $\*g(\cdot)$ is independent.
		We also note that the columns of matrix $\*G$ are also independent with each other, then according to Theorem 5.41 in~\cite{vershynin2010introduction}, with probability at least $1-2\exp(-ct^2)$, we have:
		\[
		\sigma_{\min}\left(\*G\right) \geq \sqrt{|v|} - t C_G > 0,
		\]
		where the last inequality holds when $C_G$ is small and $|v|$ is large, for convenience, we let  $\left(\sqrt{|v|} - t C_G\right) \geq \sqrt{|v|} /2$.  We can conclude that with high probability:
		\[
		\sigma_{\min}\left(		\left[
		\begin{matrix}
		\*G\\
		\*1
		\end{matrix}
		\right]	\right) \geq \sqrt{|v|} - t C_G  \geq \frac{\sqrt{|v|}}{2}>0.
		\]
		Namely, $\*{\overline{w}}$ exists and
		\[
		\*{\overline{w}} =
		\left[
		\begin{matrix}
		\*G\\
		\*1
		\end{matrix}
		\right]^{ \dagger}
		\left[
		\begin{matrix}
		f_{\bm{\theta}}(\*x)\\
		1 \\
		\end{matrix}
		\right],
		\]
		where $\*A^{\dagger}$ represents the pseudo-inverse of the matrix $\*A$. Then we have:
		\[
		\left\| \*{\overline{w}} \right\|^2 \leq \left\| \left[
		\begin{matrix}
		\*G\\
		\*1
		\end{matrix}
		\right]^{ \dagger} \right\|^2
		\left\|
		\left[
		\begin{matrix}
		f_{\bm{\theta}}(\*x)\\
		1 \\
		\end{matrix}
		\right]
		\right\|^2\leq \frac{4 C_f}{|v|},
		\]
		where $C_f = 1+ \max_{\*x \in \^C} \|f_{\bm{\theta}}(\*x)\|$, and $\|\cdot\|$ is the spectral norm for matrix and $\ell_2$-norm for vector. We finish the proof of this claim.
	\end{proof}	
	By $\*{\overline{w}}$, we can represent $f_{\bm{\theta}}(\*x)$ as a linear combination way:
	\[
	f_{\bm{\theta}}(\*x) = \sum_{i=1}^{|v|} \overline{w}_i~ g(\*x_{(i)}) := \sum_{i=1}^n\*I\left\{\*x_i \in \^V(\*x)\right\} W(\*x,\*x_i) \qty(g(\*x_i) + \varepsilon_i),
	\]
	where $W(\*x,\*x_i):=\overline{w}_i$, $W: \mathbb{R}^d_x \times \mathbb{R}^d_x \to \mathbb{R}$ is a coefficient mapping and $g(\cdot)$ is the target function.
	Note that, for any $\*x$, $ \sum_{i=1}^{|v|} \overline{w}_i = 1$ indicates:
	\begin{equation}\label{eq:sum1}
	\sum_{i=1}^n\*I\left\{\*x_i \in \^V(\*x) \right\} W(\*x,\*x_i) = 1.
	\end{equation}
	Hence, for all $ \*x \in \^C$, we can have:
	\begin{equation}\label{eq:sum_g}
		\sum_{i=1}^n \left(\*I\left\{\*x_i \in \^V(\*x) \right\} W(\*x,\*x_i) g(\*x)\right) = g(\*x).
	\end{equation}
	We consider the event:
	\[
	\^E:=\left\{\operatorname{diam}(\^C_s) \leq h \right\},
	\]
	where we specify the scale of $h$ at the last of this proof.
	Since the points $\left\{\*x_i\right\}_{i=1}^n \setminus \^V(\*x)$ are out of the subset $\^C_s$, we can observe that:
	\[
	p(\overline{\^E}) \leq \left(1-C_1p_{\min}h^{d_x}\right)^{n-|v|}\leq \exp\left\{-C_2p_{\min}nh^{d_x}\right\},
	\]
	where the last inequality comes from $|v| \ll n$ and $C_1$ and $C_2>0$ is a constant which is independent of size $n$. 
	On the event $\overline{\^E}$, due to the bounded first derivative of $g(\cdot)$ and Eq. (\ref{eq:sum_g}), with probability at least $1-2\exp(-c^2_{\varepsilon}/2)$, we have:
	\[
	\left|f_{\bm{\theta}}(\*x) - g(\*x)\right|  = \left|\sum_{i=1}^n\qty(\*I\left\{\*x_i \in \^V(\*x)\right\} W(\*x,\*x_i) \qty(g(\*x_i) + \varepsilon_i - g(\*x)))\right| \leq C_{g'} \operatorname{diam}(\^C)+ c_{\varepsilon}:=C_g.
	\]
	Thus, the contribution of event $\overline{\^E}$ to generalization bound is at most $C_g^2 \exp\left\{-Cp_{\min}h^{d_x}\right\}$, a lower-order term compared to the remaining contribution of event $\^E$.
	\par
	By the event $\^E$, we have the following decomposition:
	\[
	\mathbb{E}\left[\left|f_{\bm{\theta}}\left(\*x\right)-g\left(\*x\right)\right|^{2}\right] \leq\underbrace{\mathbb{E}\left[|f_{\bm{\theta}}\left(\*x\right)-g\left(\*x\right)|^{2}\*I\left\{\^E\right\}\right]}_{B^2(\*x)}+ C_g^2 \exp\left\{-Cp_{\min}h^{d_x}\right\},
	\]
	where $\mathbb{E}\qty[\cdot]\coloneqq \mathbb{E}_{\^{S}^n}\qty[\mathbb{E}_{\varepsilon}\qty[\cdot\mid \^S^n]]$.
	\par
	In the following, we provide the generalization bound of the bias term $B^2(\*x)$. Due to Eq. (\ref{eq:sum1}), we have:
	\[
	B^2(\*x) = \mathbb{E}\left[ \sum_{i,j=1}^n  \left(g(\*x_i) + \varepsilon_i - g(\*x)\right) \left(g(\*x_j)+ \varepsilon_j - g(\*x)\right)W_i W_j \*I\left\{\^E\right\}\right],
	\]
	where
	\[
	W_i = \*I\left\{\*x_i \in \^V(\*x)\right\} W(\*x,\*x_i).
	\]
	Due to the event $\left\{\*x_i \in \^V(\*x)\right\}$, we can conclude that $\|\*x_i - \*x\|\leq h$ and by the bounded first derivative of $g(\cdot)$, we get:
	\[
	B^2(\*x) \leq C^2_{g'} h^2 \sum_{i,j=1}^n \mathbb{E}\left[ W_i W_j \*I\left\{\^E\right\}\right] + \sum_{i,j=1}^n \mathbb{E}\left[ \varepsilon_i \varepsilon W_i W_j \*I\left\{\^E\right\}\right].
	\]
	Note that
	\[
	\sum_{i,j=1}^n \mathbb{E}\left[ \varepsilon_i \varepsilon W_i W_j \*I\left\{\^E\right\}\right] = 
	\sum_{i=1}^n\mathbb{E}\left[ W_i^2\*I\left\{\^E\right\}\right]\coloneqq \sigma_n.
	\]
	In general, according to correlation between $W_i$ and $W_j$, we decompose the sum term:
	\[
	\begin{aligned}
	& \sum_{i,j=1}^n  \mathbb{E}\left[ W_i W_j \*I\left\{\^E\right\}\right] = \underbrace{ \sum_{i=1}^n\mathbb{E}\left[ W_i^2\*I\left\{\^E\right\}\right]}_{\sigma_n} + \sum_{i\neq j}^n \mathbb{E}\left[ W_iW_j\*I\left\{\^E\right\}\right]\\
	=\quad &\mathbb{E}\left[ \sum_{i\neq j}^n W_iW_j\*I\left\{\^E\right\}\right] + \sigma_n
	\leq  \mathbb{E}\left[ (\sum_{i}^n W_i)^2\*I\left\{\^E\right\}\right]+ \sigma_n
	\leq 1 + \sigma_n,
	\end{aligned}
	\]
	where the last inequality comes from Eq. (\ref{eq:sum1}).
	On one hand, we have:
	\[
	\sigma_n = \sum_{i=1}^{|v|} \mathbb{E}\left[ W(\*x,\*x_{(i)})^2 \*I\left\{\^E\right\}\right] =  \mathbb{E}\left[ \|\*{\overline{w}}\|^2 \*I\left\{\^E\right\}\right].
	\]
	Actually, the random variables $\*I\left\{\*x_i \in \^V(\*x)\right\}$ follow the Bernoulli distribution with parameter:
	\[
	\hat{p}:= P({\*x_i \in \^V(x) }) \geq c_0 p_{\min} h^{d_x},
	\]
	where $c_0 > 0$ depends on the shape of set $\^C_s$ and $d$.
	Hence we can divide the exception into two term:
	\[
	\mathbb{E}\left[ \|\*{\overline{w}}\|^2 \*I\left\{\^E\right\}\right]
	\leq\quad \underbrace{\mathbb{E}\left[ \|\*{\overline{w}}\|^2 \*I\left\{\^E\right\}\*I\left\{|v|<\frac{n\hat{p}}{2}\right\}\right]}_{E_1} + \underbrace{\mathbb{E}\left[ \|\*{\overline{w}}\|^2 \*I\left\{\^E\right\}\*I\left\{|v|\geq \frac{n\hat{p}}{2}\right\}\right]}_{E_2}.
	\]
	For $E_2$, together with Eq. (\ref{eq:boundW}), we have:
	\[
	\begin{aligned}		
	E_2
	\leq&\quad p_{\max} \frac{C_w}{|v|} \int_{\^C_s} \*I\left\{\^E\right\} d\*x
	\leq\quad c p_{\max} \frac{2C_w}{n\hat{p}}  h^{d_x} \int_{0}^1 r^{d_x -1} dr\\
	\leq &\quad c p_{\max} \frac{2C_w}{c_0 n p_{\min}}:= \frac{c_1}{n},
	\end{aligned}
	\]
	where $c>0$ is the constant which is independent of $n$ and depends on the shape of the set $\^C_s$.
	For $E_1$, we have
	\[
	\begin{aligned}	
	E_1
	\leq & \quad p_{\max} \frac{C_w}{|v|} \mathbb{E}\left[ \*I\left\{|v|<\frac{n\hat{p}}{2}\right\}\right]
	\leq  \quad p_{\max} \frac{C_w}{d_x}  P\left(\sum_{i=1}^n \*I\left\{\*x_i \in \^V(\*x)\right\} < \frac{n\hat{p}}{2}\right)\\
	= & \quad p_{\max} \frac{C_w}{d_x}  P\left( \left|\sum_{i=1}^n \*I\left\{\*x_i \in \^V(\*x)\right\} - n\hat{p}\right| > \frac{n\hat{p}}{2} \right)\\
	\overset{(a)}{\leq}& \quad p_{\max} \frac{C_w}{d_x} \exp\left\{\frac{(n\hat{p}/2)^2}{2n\hat{p}(1-\hat{p})+ n\hat{p} /3}\right\}
	\leq   \exp\left\{-c_2 n h^{d_x}\right\},
	\end{aligned}
	\]
	where $(a)$ comes from the Bernstein’s inequality.
	\par
	Combing all the above results together, by setting $h = \order{n^{-\frac{1}{d_x+2\beta}}}$, we obtain:
	\[
	\begin{aligned}
	\mathbb{E}\left[|f_{\bm{\theta}}\left(\*x\right)-g\left(\*x\right)|^{2}\right]
	\leq& \quad  C_g^2 \exp\left\{-C_2p_{\min}nh^{d_x}\right\} +  C^2_{g'} h^2 \cdot\left( 1 + \frac{c_1}{n}+ \exp\left\{-c_2 n h^{d_x}\right\} \right) + \qty(\frac{c_1}{n}+ \exp\left\{-c_2 n h^{d_x}\right\})\\
	\leq&\quad C_3\exp\left\{-C_4 nh^{d_x}\right\} + C_5 h^{2} + \frac{c_1}{n}\\
	\leq&\quad \frac{C_3C_4}{nh^{d_x}} + C_5h^2 + \frac{c_1}{n} \leq~ C_6 n^{\left(-\frac{2}{2+d_x}\right)},
	\end{aligned}
	\]
	where  $\{C_3, C_4, C_5, C_6\}>0$ are  universal constants and the last inequality holds when $h = \order{n^{-\frac{1}{d_x+2\beta}}}$. 
	It is obvious that, when $n$ is large enough and the data is sampled uniformly, the event $\^E:=\left\{\operatorname{diam}(\^C_s) \leq h \right\}$ can easily happen for $h = \order{n^{-\frac{1}{d_x+2\beta}}}$.
	\par
	We now finish the proof of this theorem.
\end{proof}

\subsection{Proof of Theorem \ref{thm:full}}
\begin{proof}
	Actually, the proof is very direct.
	Let $\ell_{\Phi}(\cdot):=\ell(\Phi(\cdot), \*y)$ and $\nabla \ell_{\Phi}(\*h(\*x))$ be the gradient $\nabla \ell_{\Phi}$ evaluated at $\*h(\*x)$. 
	Denote by $\bm{\theta}_{0}$ the parameters of $\*h_0$. 
	Note that $\*h(\cdot)$ the DNN appended with $l$-layer MCN has the parameters $[\bm{\theta}_{0}, \bm{\theta}_{l}]$.
	Given the local minimum $\left[\bm{\tilde{\theta}}_{0}, \bm{\tilde{\theta}}_{l+1}\right]$ and the parameters $\bm{\theta}'_{0}$ such that $\*h_0(\cdot\mid \bm{\theta}'_{0})$ is injective w.r.t to the input $\*x$, then we have
	\[
	\frac{1}{n}\sum_{i=1}^n \ell_{\Phi}\left(\*h\left(\*x_i\mid \left[\bm{\tilde{\theta}}_{0}, \bm{\tilde{\theta}}_{l+1}\right]\right)\right) \leq
	\min_{\left[\bm{\theta}_{0}, \bm{\theta}_{l}\right]} 	\frac{1}{n}\sum_{i=1}^n \ell_{\Phi}\left(\*h(\*x_i)\right) 
	\leq \min_{\bm{\theta}_{l}} 
	\frac{1}{n}\sum_{i=1}^n \ell_{\Phi}\left(\*h\left(\*x_i \mid \bm{\theta}'_{0}\right)\right),
	\]
	where the first inequity comes from Theorem \ref{thm:depth}.
	It is obvious the right side in the above inequality is the loss of a $l$-layer MCN with the set $\{\left(\*h(\*x_i\mid \bm{\theta}'_{0}),y_i\right)\}_{i=1}^n$ at the global minimum. The problem becomes a learning target with the input as $\*h(\*x_i\mid \bm{\theta}'_{0})$. As shown in Claim \ref{clm:interpolation}, a $(n-1)$-th order polynomial can exactly fit the training set. 
	For a given polynomial, it is easy to modified it to make it satisfy the Condition \ref{cond:r2f}, e.g., extending and rescaling. With the virtue of Theorem \ref{thm:approximation}, MCN can approximate the functions satisfing Condition \ref{cond:r2f} arbitrarily well as it goes deeper and wider. Hence, we have
	\[
	\min_{\bm{\theta}_{l}} 
	\frac{1}{n}\sum_{i=1}^n \ell_{\Phi}\left(\*h(\*x_i \mid \bm{\theta}'_{0})\right) \to 0, \quad l \to \infty.
	\] 
	Thus,
	\[
	\frac{1}{n}\sum_{i=1}^n \ell_{\Phi}\left(\*h\left(\*x_i\mid \left[\bm{\tilde{\theta}}_{0}, \bm{\tilde{\theta}}_{l+1}\right]\right)\right) \to 0
	\]
	holds at any local minimum $\left[\bm{\tilde{\theta}}_{0}, \bm{\tilde{\theta}}_{l+1}\right]$ as MCN goes deeper and wider.
\end{proof}
%

\section{Connection to Linear Regression}\label{sec:connectLR}
In this section, we shall quantitatively describe the quality of each local minimum on the regression task. Denote $\^P_{\*D}$ as the orthogonal projection matrix onto the column space (or range space) of a matrix $\*D$, thereby $\^P^\perp_{\*D} = \*I - \^P_{\*D}$. Let $\otimes$ represent the Kronecker product, let $\operatorname{vec}(\cdot)$ be the vectorization of a matrix, and denote the $d_y$-dimension identify matrix as $\*I_{d_y}$. Denote by $\*Y := [\*y_1,\cdots,\*y_n]$ the target matrix. With these notations, we have the following theorem to measure the training objective quantitatively.
\begin{theorem}[Monotonicity of Objective]\label{thm:projet}
	Suppose that $\bm{\theta}_l$ is a local minimum to problem (\ref{eq:tr_loss}), in which the loss $\ell$ is chosen as the squared loss and the mapping $\Psi(\cdot)$ is a learnable matrix of size $d_y\times{}d_l$. Then the following holds:
	\begin{enumerate}[label=(\roman*)]
		\item There exists a matrix $\*D$ whose column space expands, as the depth and width of MCN increase; and
		\[L(\bm{\theta}_l) = \frac{1}{n}\|\^P^\perp_{\*D} \operatorname{vec}\left(\*Y\right)\|^2.\]
		
		\item For any $k \in [l]$ and $i \in [n]$, if $\^W_{k}(\*x_{k-1,i})$ is independent with the first $d_{\^L}$ dimension of the input $\*x_{k-1,i}$ then
		\[
		L(\bm{\theta}_l) = \underbrace{\frac{1}{n}\|\^P^\perp_{\widehat{\*D}} \operatorname{vec}\left(\*Y\right)\|^2}_{
			\begin{subarray}{l}
			\text{\footnotesize{global optimum value of linear}} \\
			\text{\footnotesize{regression with basis matrix}$~\widehat{\*D}$}
			\end{subarray}},	
		\]
		where $\*D$ is the same with (i),
		$
		\widehat{\*X} :=
		\begin{bmatrix}
		\*X_1\otimes \*I_{d_y}& \*X_2\otimes \*I_{d_y} &\cdots& \*X_l\otimes \*I_{d_y}
		\end{bmatrix}^\top,
		\*X_k :=
		\begin{bmatrix}
		\*x_{k,1}&\*x_{k,2}&&\cdots&\*x_{k,n}
		\end{bmatrix},~\forall k \in [l]
		$ and
		$\widehat{\*D}:=	\begin{bmatrix}
		\widehat{\*X}&\*D
		\end{bmatrix}$.
	\end{enumerate}
\end{theorem}
Theorem~\ref{thm:projet} is applicable to a wide range of DNNs, ranging from under-parameterized shallow networks to over-parameterized deep architectures. It makes connections between the training objective of MCN and the global minimum value of linear regression, in which the basis matrix is composed of the network parameters and the outputs of hidden layers. When the MCN architecture goes deeper and wider, the column space of $\widehat{\*D}$ expands and thus $\^P^\perp_{\*D} \operatorname{vec}\left(\*Y\right)$ deflates and, accordingly, the training objective may decrease. In other words, for the squared regression problems, the training performance of MCN becomes better as the depth increases even in the worst scenario. So for our MCN, it is the deeper the better.
\subsection{Proof of Theorem \ref{thm:projet}}
\begin{proof}
	Similar to the proof Theorem {\ref{thm:depth}}, we simplify MCN as:
	\[
	\*x_{k+1,i} = \left[\*L_{k+1}\*x_{k,i};~\*{\tilde{A}}_{k+1}\*x_i +\max\left\{\*W_{k+1}\*x_{k,i},  \sigma\left(\*A_{k+1}\*x_i\right)\right\}\right].
	\]
	In this section, we denote the linear transformation of the output of MCN $\*Y_{\bm{\theta}} := [\Psi\left(\*f_{\bm{\theta}}(\*x_1)\right),\cdots,\Psi\left(\*f_{\bm{\theta}}(\*x_n)\right)] \in \mathbb{R}^{d_y \times n}$. Denote the target matrix as $\*Y := [\*y_1,\cdots,\*y_n]$ and the training data as $\*X:=[\*x_1,\cdots,\*x_n]$. Denote by $\*X_k:=[\*x_{k,1},\cdots,\*x_{k,n}]$ the output of the $k$-th layer.
	\par
	Geven $\bm{\theta}_k$ as a local minimum of the loss function $L$, we define several notations. First, we define a mask operator $\hat{\*\Lambda}_{k}$, for $(j_1,j_2) \in [d_k]\times [n]$, such that:
	\[
	\left(\hat{\*\Lambda}_{k}\right)_{(j_1,j_2)}:=
	\begin{cases}
	1,& \quad \text{if}\quad \left(\*W_{k}\*x_{k-1,i}\right)_{(j_1,j_2)}  \geq  \left( \sigma\left(\*A_{k}\*x_i\right)\right)_{(j_1,j_2)};\\
	0,& \quad \text{otherwise}.
	\end{cases}
	\]
	We also define:
	\[
	\tilde{\*\Lambda}_k = \operatorname{diag}\left(
	\begin{bmatrix}
	\*1\\
	\operatorname{vec}(\hat{\*\Lambda}_{k})\\
	\end{bmatrix}\right),
	\]
	where $\*1 \in \mathbb{R}^{d_{\^L}n}$ is the all one vector and $\operatorname{diag}$ is the diagonal operator. We denote the complementary matrix of $\tilde{\*\Lambda}_l$ as $\tilde{\*\Lambda}_l^{\perp} := \*1 -\tilde{\*\Lambda}_l$, where $\*1$ is the all one matrix with the compatibility dimension. Let
	\[
	\widetilde{\*W}_{k} =
	\left[\begin{matrix}
	\*L_{k}\\
	\*W_k\\
	\end{matrix}\right],\quad
	\*b_{k} =
	\begin{bmatrix}
	\*0\\
	\operatorname{vec}\left(\sigma\left(\*A_{k}\*X\right)\right)\\
	\end{bmatrix} \quad\text{and}\quad
	\*c_{k} =
	\begin{bmatrix}
	\*0\\
	\operatorname{vec}\left(\*{\tilde{A}}_{k}\*X\right)\\
	\end{bmatrix},
	\]
	where where $\*0 \in \mathbb{R}^{d_{\^L}n}$ is the all zero vector.
	Since $\Psi(\cdot)$ is a learnable linear operator, for brevity, we denote $\Psi(\*x_{l,i})$ as $\Psi(\*x_i)\*x_{l,i}$, and let:
	\[
	\*C_{l+1} :=
	\begin{bmatrix}
	\Psi(\*x_i) & & \\
	& \ddots & \\
	& & \Psi(\*x_n)
	\end{bmatrix}\cdot \tilde{\*\Lambda}_{l },\quad
	\*C_{k+1} :=
	\left(\*I_n \otimes \widetilde{\*W}_{k+1}\right) \tilde{\*\Lambda}_{k},
	\]
	and
	\[
	\*C_{l+1}' :=
	\begin{bmatrix}
	\Psi(\*x_i) & & \\
	& \ddots & \\
	& & \Psi(\*x_n)
	\end{bmatrix}\cdot \tilde{\*\Lambda}_{l}^\perp,\quad
	\*C_{k+1}' :=
	\left(\*I_n \otimes \widetilde{\*W}_{k+1}\right) \tilde{\*\Lambda}_{k}^\perp.
	\]	
	\par
	With these notations, we can have the following two claims.
	\begin{claim}\label{clm:pjC}
		For all $k \in [l]$, and,we have:
		\[
		\partial_{\widetilde{\*W}_{k}} \*Y_{\bm{\theta}}=\*D_{k},
		\]
		where
		\[
		\*D_{k} = \*C_{l+1}\cdots \*C_{k+1} \left(\*X_{k-1}^\top \otimes \*I_{d_k} \right).
		\]
	\end{claim}
	\begin{proof}
		We can rewrite MCN as the vectorized form:
		\[
		\begin{aligned}
		\operatorname{vec}(\*X_{k}) = &~ \tilde{\*\Lambda}_k\operatorname{vec}(\widetilde{\*W}_k \*X_{k-1}) + \tilde{\*\Lambda}_k^{\perp}\cdot
		\operatorname{vec}\left(
		\begin{bmatrix}
		\*0\\
		\sigma\left(\*A_{k}\*X\right)\\
		\end{bmatrix}
		\right)+
		\operatorname{vec}\left(
		\begin{bmatrix}
		\*0\\
		\*{\tilde{A}}_{k}\*X\\
		\end{bmatrix}
		\right)
		\\
		=& ~\tilde{\*\Lambda}_k \left(\*I_n \otimes \widetilde{\*W}_k\right)
		\operatorname{vec}\left(\*X_{k-1}\right) + \tilde{\*\Lambda}_k^{\perp} \*b_k + \*c_k\\
		=&~ \tilde{\*\Lambda}_k \left(\*X_{k-1}^\top \otimes \*I_{d_k} \right)
		\operatorname{vec}\left( \widetilde{\*W}_k\right) + \tilde{\*\Lambda}_k^{\perp} \*b_k + \*c_k.
		\end{aligned}
		\]
		By the definition of $\*C$ and $\*C'$, we have:
		\begin{equation}\label{eq:vecY}
		\operatorname{vec}(\*Y_{\bm{\theta}}) = \left(\stackrel{\leftarrow }{\prod_{k'=k}^{l}} \*C_{k'+1}\right) \left(\*X_{k-1}^\top \otimes \*I_{d_k} \right) \operatorname{vec}\left( \widetilde{\*W}_k\right) + \sum_{j=k}^{l}\left(\stackrel{\leftarrow}{\prod_{k'=j+2}^{l+1}} \*C_{k'}\right)\left(\*C_{j+1}' \*b_j + \*c_j'\right),	
		\end{equation}
		where we let
		\[
		\*c_j' = \left(\*I_n \otimes \widetilde{\*W}_{j+1}\right)\*c_j,\quad
		\stackrel{\leftarrow }{\prod_{k'=l+2}^{l+1}} \*C_{k'} = \*I\quad \text{and}\quad
		\stackrel{\leftarrow }{\prod_{k'=k}^{l}} \*C_{k'+1} = \*C_{l+1} \cdots \*C_{k+1}.
		\]
		We finish the proof of this claim.
	\end{proof}
	\begin{claim}\label{clm:YXzero}
		For all $l \in [L]$ and $i \in [n]$, if $\^W_{l}(\*x_{k-1,i})$ is independent with the first $d_{\^L}$ dimension of the input $\*x_{k-1,i}$ and $\Psi(\*x_i)$ is a learnable matrix, then we have:
		\[
		\left(\*Y_{\bm{\theta}} - \*Y\right) \*X_{l}^\top = \*0.
		\]
	\end{claim}
	\begin{proof}
		Since $\^W_{k}(\*x_{k-1,i})$ is independent to the first $d_{\^L}$ dimension of the input $\*x_{k-1,i}$, we can rewrite $\widetilde{\*W}_k$ as:
		\[
		\begin{bmatrix}
		\*F_k & \*G_k\\
		0& \*H_k
		\end{bmatrix}
		:= \widetilde{\*W}_k.
		\]
		where $	\begin{bmatrix}
		\*F_k & \*G_k
		\end{bmatrix} := \*L_{k}$ and $\*F_k \in \mathbb{R}^{d_{\^L} \times d_{\^L}}$.
		We can have:
		\[
		\begin{aligned}	
		\*X_{k+1} = &~
		\begin{bmatrix}
		\*1\\
		\hat{\*\Lambda}_{k+1}
		\end{bmatrix} \circ
		\widetilde{\*W}_{k+1}
		\begin{bmatrix}
		\*L_k \*X_{k-1}\\
		\overline{\*X}_k
		\end{bmatrix} +
		\begin{bmatrix}
		0\\
		\tilde{\*\Lambda}_{k+1}^\perp \circ \left(\sigma\left(\*A_{k+1}\*X\right)\right)
		\end{bmatrix}+
		\begin{bmatrix}
		0\\
		\*{\tilde{A}}_{k+1}\*X
		\end{bmatrix}
		\\
		= &~
		\begin{bmatrix}
		\*F_{k+1}\*L_k \*X_{k-1} + \*G_{k+1} \overline{\*X}_k\\
		\tilde{\*\Lambda}_{k+1} \circ \*H_{k+1}	\overline{\*X}_k+ \tilde{\*\Lambda}_{k+1}^\perp \circ
		\left( \sigma\left(\*A_{k+1}\*X\right)\right)
		\end{bmatrix}+
		\begin{bmatrix}
		0\\
		\*{\tilde{A}}_{k+1}\*X
		\end{bmatrix}
		\\
		= &~
		\begin{bmatrix}
		\*F_{k+1}\*L_k \*X_{k-1} + \*G_{k+1} \overline{\*X}_k\\
		\overline{\*X}_{k+1}
		\end{bmatrix}
		\end{aligned},
		\]
		where $\overline{\*X}_{k}$ is the lower $(d_k - d_{\^L})$-row part of the matrix $\*X_{k}$. Note that terms $\*G_{k+1} \overline{\*X}_k$ and $\overline{\*X}_{k+1}$ are independent with the learnable matrix $\*L_k$.
		\par	
		Without loss of generality, for all $\*x_i \in \{\*x_i\}_{i=1}^n$, we let $\begin{bmatrix}
		\*F_{k+1} & \*G_{k+1}
		\end{bmatrix} := \Psi(\cdot)$, then we can get:
		\begin{equation}\label{eq:YbyX}
		\*Y_{\bm{\theta}} = \left(\stackrel{\leftarrow }{\prod_{l'=k+2}^{l+1}} \*F_{l'}\right)\*L_{k+1}\*X_{k} + \sum_{j=k}^{l-1}\left(\stackrel{\leftarrow }{\prod_{l'=j+3}^{l+1}} \*F_{l'}\right)\*G_{k+2} \overline{\*X}_{k+1},
		\end{equation}
		
		where $$\stackrel{\leftarrow }{\prod_{l'=l+2}^{l+1}} \*F_{l'} = \*I.$$
		Note that:
		\[L(\bm{\theta}) = \frac{1}{n} \|\*Y_{\bm{\theta}}-\*Y\|_F^2.\]
		By the first order condition of the local minimum, we have:
		\begin{equation}\label{eq:firstOrder}
		\*0 = \partial_{\*L_{k+1}}L(\bm{\theta})  = \left(\*F_{l+1}\cdots \*F_{k+2}\right)^\top\left(\*Y_{\bm{\theta}}-\*Y\right)\*X_l^\top.
		\end{equation}
		If $\left(\*F_{l+1}\cdots \*F_{k+2}\right) \in \mathbb{R}^{d_{\^L}\times d_{\^L}}$ is full rank, then we finish this proof. Hence, in the rest of this proof, we consider the case:
		\[
		\operatorname{rank}\left(\stackrel{\leftarrow }{\prod_{l'=k+2}^{l+1}} \*F_{l'}\right)<d_{\^L}.
		\]
		Choosing a unit length vector from the null space of matrix $\left(\*F_{l+1}\cdots \*F_{k+2}\right)^\top$, i.e.,
		\[
		\|\*u_{k+1}\|=1,\quad\*u_{k+1} \in \operatorname{null}\left({\prod_{l'=k+2}^{1+1}} \*F_{l'}^\top\right) \subset \mathbb{R}^{d_{\^L}},
		\]
		where $\operatorname{null}(\cdot)$ denotes the null space of a matrix.
		\par
		For any $\*v_{k+1} \in \mathbb{R}^{d_k}$, we have:
		\[
		\*Y_{\bm{\theta}} = {\*Y}_{\widetilde{\bm{\theta}}}:=  \left(\stackrel{\leftarrow }{\prod_{l'=k+2}^{l+1}} \*F_{l'}\right)\widetilde{\*L}_{k+1}\*X_{k} + \sum_{j=k}^{l-1}\left(\stackrel{\leftarrow }{\prod_{l'=j+3}^{l+1}} \*F_{l'}\right)\*G_{k+2} \overline{\*X}_{k+1},
		\]
		where
		\[
		\widetilde{\*L}_{k+1} = \*L_{k+1} + \*u_{k+1}\*v_{k+1}^\top,  \quad \widetilde{\bm{\theta}} = \{\bm{\theta}\setminus \*L_{k+1}, \widetilde{\*L}_{k+1} \}.
		\]
		Since $\*Y_{\bm{\theta}} = \*Y_{\widetilde{\bm{\theta}}}$, for any sufficient small $\*v_{k+1}$, we can conclude that $\widetilde{\bm{\theta}}$ is also a local minimum of the loss function $L$. Similar to the Eq. (\ref{eq:firstOrder}), we have
		\[
		\*0 = \partial_{\*F_{k+1}}L(\widetilde{\bm{\theta}})  = \left(\*Y_{\bm{\theta}}-\*Y\right)\*X_k^\top\widetilde{\*L}_{k+1}^\top\left({\prod_{l'=k+2}^{l}} \*F_{l'}^\top\right).
		\]
		Together with $\*0 = \partial_{\*F_{k+1}}L({\bm{\theta}})$, we can have
		\begin{equation}\label{eq:induc_L}
		\*0 = \left(\*Y_{\bm{\theta}}-\*Y\right)\*X_k^\top \left(\*v_{k+1}\*u_{k+1}^\top\right)\left({\prod_{l'=k+2}^{l}} \*F_{l'}^\top\right).
		\end{equation}
		We now show that,
		\begin{equation}\label{eq:induc}
		\*0 = \left(\*Y_{\bm{\theta}}-\*Y\right)\*X_k^\top \left(\*v_{k+1}\*u_{k+1}^\top\right)\left({\prod_{l'=k+2}^{j}} \*F_{l'}^\top\right),
		\end{equation}
		by induction on the index $j = \{l,l-1,\cdots,k+1\}$. The base case $j=l$ is proven above. We consider the case that $j = l-1$. If $\*F_{l+1}$ is full rank, then we have:
		\[
		\*0 = \partial_{\*F_{l}}L(\widetilde{\bm{\theta}})  = \*F_{l+1}^\top\left(\*Y_{\bm{\theta}}-\*Y\right)\*X_k^\top\widetilde{\*L}_{k+1}^\top\left(\prod_{l'=k+2}^{l-1} \*F_{l'}^\top\right),
		\]
		which indicates that Eq. (\ref{eq:induc}) holds for $j = l-1$. Now we assume $\operatorname{rank}(\*F_{l+1}) < d_{\^L}$.
		Similarly, choosing a unit length vector from the null space of matrix $\*F_{l+1}$, i.e.,
		\[
		\|\*u_{l}\|=1,\quad\*u_{l} \in \operatorname{null}\left(\*F_{l+1}\right) \subset \mathbb{R}^{d_{\^L}}.
		\]
		Define:
		\[
		\widetilde{\*F}_{l} = \*F_{l} + \*u_{l}\*v_{l}^\top,  \quad \widetilde{\bm{\theta}}' = \{\widetilde{\bm{\theta}}\setminus \*F_{l}, \widetilde{\*F}_{l} \},
		\]
		where $\*v_l \in \mathbb{R}^{d_{\^L}}$. 	
		Similarly, we can get $\*Y_{\bm{\theta}} = \*Y_{\widetilde{\bm{\theta}}'}$. Hence, for any sufficient small $\*v_{l}$, we can conclude that $\widetilde{\bm{\theta}}'$ is also a local minimum of the loss function $L$, then:
		\[
		\*0 = \partial_{\*F_{l+1}}L(\widetilde{\bm{\theta}}')  = \left(\*Y_{\bm{\theta}}-\*Y\right)\*X_k^\top\widetilde{\*L}_{k+1}^\top\left({\prod_{l'=k+2}^{l-1}} \*F_{l'}^\top\right)\widetilde{\*F}_{l}.
		\]
		Together with $\*0 = \partial_{\*F_{l+1}}L({\bm{\theta}})$ and Eq. (\ref{eq:induc_L}), we can have:
		\[
		\*0 = \left(\*Y_{\bm{\theta}}-\*Y\right)\*X_k^\top\widetilde{\*L}_{k+1}^\top \left({\prod_{l'=k+2}^{l}}  \*F_{l'}^\top\right)(\*v_{l}\*u_{l}^\top).
		\]
		Notice that we can easily have
		\[
		\*0 = \left(\*Y_{\bm{\theta}}-\*Y\right)\*X_k^\top{\*L}_{k+1}^\top \left({\prod_{l'=k+2}^{l}}  \*F_{l'}^\top\right)(\*v_{l}\*u_{l}^\top),
		\]
		by setting $\widetilde{\bm{\theta}}' = \{{\bm{\theta}}\setminus \*F_{l}, \widetilde{\*F}_{l} \}$ and using the first order condition w.r.t. the matrix $\*F_{l+1}$. Thus, we can conclude:
		\[
		\*0 = \left(\*Y_{\bm{\theta}}-\*Y\right)\*X_k^\top \left(\*v_{k+1}\*u_{k+1}^\top\right)\left({\prod_{l'=k+2}^{l}} \*F_{l'}^\top\right)(\*v_{l}\*u_{l}^\top),
		\]
		which also implies
		\[
		\*0 = \left(\*Y_{\bm{\theta}}-\*Y\right)\*X_k^\top \left(\*v_{k+1}\*u_{k+1}^\top\right)\left({\prod_{l'=k+2}^{l}} \*F_{l'}^\top\right)\*v_{l}.
		\]
		The above enquality holds for all sufficient small $\*v_l$. We can conclude that Eq. (\ref{eq:induc}) holds for $j = l-1$. This completes the inductive step and proves that:
		\[
		\*0 = \left(\*Y_{\bm{\theta}}-\*Y\right)\*X_k^\top \left(\*v_{k+1}\*u_{k+1}^\top\right),
		\]
		which obviously implies:
		\[
		\*0 = \left(\*Y_{\bm{\theta}}-\*Y\right)\*X_k^\top.
		\]	
		We now finish the proof of this claim.
	\end{proof}
	\paragraph{Proof of Theorem \ref{thm:projet} (\RNum{1})} From the first order necessary condition of differentiable local minima, we have:
	\[
	\*0 = \partial_{\widetilde{\*W}_{k}}L({\bm{\theta}}) = \*D_k^\top \operatorname{vec}\left(\*Y_{\bm{\theta}} - \*Y\right),
	\]
	where the last equation comes from Claim \ref{clm:pjC}. Let
	\[
	\*D :=
	\begin{bmatrix}
	\*D_1&\*D_2&\cdots&\*D_{l+1}
	\end{bmatrix}.
	\]
	We have
	\[
	\*0 = \*D^\top \operatorname{vec}\left(\*Y_{\bm{\theta}} - \*Y\right).
	\]
	According to the Eq. (\ref{eq:vecY}), it is obvious that $\operatorname{vec}\left(\*Y_{\bm{\theta}}\right)$ belongs to the column space of matrix $\*D$. Thus, we can conclude:
	\[
	\operatorname{vec}\left(\*Y_{\bm{\theta}}\right) = \^P_{\*D} \operatorname{vec}\left(\*Y\right).
	\]
	Therefore,
	\[
	nL(\bm{\theta}) = \|\*Y_{\bm{\theta}} - \*Y\|_F^2 = \| \operatorname{vec}\left(\*Y_{\bm{\theta}} - \*Y\right)\|^2
	= \|\^P_{\*D} \operatorname{vec}\left(\*Y\right) - \operatorname{vec}\left(\*Y\right)\|^2 = \|\^P^\perp_{\*D} \operatorname{vec}\left(\*Y\right)\|^2.
	\]
	\paragraph{Proof of Theorem \ref{thm:projet} (\RNum{2})} From Claim \ref{clm:YXzero}, we have:
	\[
	\*0 = \left(\*X_k \otimes \*I_{d_y}\right) \operatorname{vec}\left(\*Y_{\bm{\theta}} - \*Y\right),
	\]
	Let
	\[
	\widehat{\*X} :=
	\begin{bmatrix}
	\*X_1\otimes \*I_{d_y} &\*X_2 \otimes \*I_{d_y} &\cdots& \*X_l\otimes \*I_{d_y}
	\end{bmatrix}^\top
	\]
	We have
	\[
	\*0 = \widehat{\*X}^\top \operatorname{vec}\left(\*Y_{\bm{\theta}} - \*Y\right).
	\]
	Combine the result of Claim \ref{clm:pjC}, we can get:
	\[
	\*0 =
	\begin{bmatrix}
	\widehat{\*X}& \*D
	\end{bmatrix}^\top \operatorname{vec}\left(\*Y_{\bm{\theta}} - \*Y\right).
	\]
	According to the Eq. (\ref{eq:YbyX}), it is obvious that $\operatorname{vec}\left(\*Y_{\bm{\theta}}\right)$ belongs to the column space of matrix $\widehat{\*X}$. Thus, we can conclude:
	\[
	\operatorname{vec}\left(\*Y_{\bm{\theta}}\right) = \^P_{\begin{bmatrix}
		\widehat{\*X}& \*D
		\end{bmatrix}} \operatorname{vec}\left(\*Y\right).
	\]
	Therefore,
	\[
	\begin{aligned}
	nL(\bm{\theta}) =&~ \| \operatorname{vec}\left(\*Y_{\bm{\theta}} - \*Y\right)\|^2
	= \|\^P_{\begin{bmatrix}
		\widehat{\*X}& \*D
		\end{bmatrix}}  \operatorname{vec}\left(\*Y\right) - \operatorname{vec}\left(\*Y\right)\|^2 \\
	=&~ \|\^P_{\widehat{\*X}} \operatorname{vec}\left(\*Y\right) -\operatorname{vec}\left(\*Y\right) + \^P_{[\^P^\perp_{\widehat{\*X}}\cdot\*D]}\operatorname{vec}\left(\*Y\right) \|^2\\
	=&~ \|\^P_{[\^P^\perp_{\widehat{\*X}}\cdot\*D]}\operatorname{vec}\left(\*Y\right)-\^P^\perp_{\widehat{\*X}} \operatorname{vec}\left(\*Y\right)\|^2\\
	=&~\|\^P^\perp_{\widehat{\*X}} \operatorname{vec}\left(\*Y\right)\|^2-
	\|\^P_{[\^P^\perp_{\widehat{\*X}}\cdot\*D]} \operatorname{vec}\left(\*Y\right)\|^2.
	\end{aligned}
	\]
	We now finish the whole proof.
\end{proof}

\section{Prior Arts}
\subsection{Effects of Depth and Width in Neural Networks}
Usually, each layer of wide networks contains abundant hidden units, and these units can be seen as one kind of features. Hence, wide networks~(even infinitely wide) naturally have connection with the kernels and Gaussian processes. By the kernel methods, the works in~\cite{xie2016diverse,du2018gradient2} lower bounded the spectrum of Gram matrix and revealed that the network learning is actually a regression problem, but their theoretical bounds only hold for shallow networks. Then works~\cite{du2018gradient1,arora2019exact} captured the behavior of fully-connected deep networks in the large (maybe infinite) width limit trained by gradient descent and also found the equivalence between the kernel regression predictor and wide networks. However, all these works do not show the benefits of depth, and deeper nets do not obtain better theoretical results than shallow ones in their settings.
\par
Depth is also important to the general networks. Generally, a neural network with $\Theta(k^3)$ layers, $\Theta(1)$ units per layer, cannot be approximated by networks with $O(k)$ layers~\cite{telgarsky2016benefits}. The works~\cite{kawaguchi2019effect, arora2018convergence} showed that deeper and wider fully-connected networks obtain better training results, but did not analyze the NN's performance during testing. By contrast, besides showing the training objective decreases monotonously with the increase of depth and width, we also give the generalization bound of the proposed MCN. In addition, we prove that $(l+1)$-layer MCN always obtains better training results than  $l$-layer MCN, which reveals the reason why deeper nets usually perform better in practice.
\subsection{Generalization of Neural Networks}
One major concern in the learning community is the generalization bound~(also known as estimation bound). In general, at least $n = \Omega(\epsilon^{-\max\{d_x,2\}})$ samples are needed to learn a Lipschitz-continuous functions in $\mathbb{R}^{d_x}$ with the population regression risk as $\epsilon$~\cite{luxburg2004distance}. The exponential dependence on the dimension $d_x$ is often referred to as the \emph{curse of dimensionality}. Fortunately, when the model structure is specified, the sample complexity can be reduced, e.g., $\Omega(d_x\epsilon^{-2})$ for affine functions~\cite{shalev2014understanding}, $\Omega(k^2d_x\epsilon^{-2})$ for single hidden-layer fully connected neural networks \cite{rumerlhar1986learning}, where $k$ is the number
of units in the hidden layer, and $\widetilde{O}\left(m^2 \epsilon^{-4}\right)$ for one-hidden-layer CNN with $m$-dimensional convolutional filter~\cite{du2018many}.
Generally, for a parametric regression problem, the expected generalization
error is bounded as $O(D \log(n)/n)$, where $D$ is depends on the amount of model parameters~\cite{maillard2009compressed, gyorfi2006distribution}. Obviously, this bound cannot reveal the mystery of generalization ability of over-parametrized deep learning models which have more parameters than necessary to fit the training data.

In practice, we first train DNNs to perfectly fit the training data.
The resulting (zero training loss) NNs can already have good performance on test data~\cite{zhang2016understanding}.
This phenomena is considered as one of reasons to concern the theoretical  generalization bound of neural networks. Some researchers try to find the inspiration from shallow networks. By assuming the existence of a true model, the works in~\cite{ma2018priori,du2018many,arora2019fine} showed that the (regularized) empirical risk minimizer has good generalization with sample complexity that depends on the true model.
Another line of researchers take the dynamic optimization process (e.g., SGD) into consideration and/or connect the network learning with kernel methods~\cite{arora2019fine, allen2018learning, dou2019training}. Although the theory is rigorous, all the works cannot be easily extended to the networks with complex structure which may not be trained by SGD.

Surprisingly, some recent works found that data interpolation also have good generalization ability and even can obtain the statistical sub-optimality and optimality for linear and kernel-based combination of observation, respectively~\cite{belkin2018overfitting,pmlr-v89-belkin19a}. Moreover, bias-variance trade-off theory for interpolating
predictors was also explored~\cite{belkin2018reconciling}. However, all these works are non-parametric and may not directly apply to the DNN analysis.
\end{document}